\pdfoutput=1
\documentclass{article}

\PassOptionsToPackage{numbers,compress}{natbib}
\usepackage[final,main]{style/neurips_2026}
\usepackage{multirow}
\usepackage{pifont}

\usepackage{multirow}

\usepackage[utf8]{inputenc}
\usepackage[T1]{fontenc}
\usepackage{adjustbox}

\usepackage{url}\usepackage{booktabs}
\usepackage{wrapfig}
\usepackage{amsfonts}
\usepackage{nicefrac}
\usepackage{silence}
\usepackage[nopatch={footnote,showhyphens}]{microtype}
\usepackage{amsmath}
\usepackage{amssymb}
\usepackage{mathtools}
\usepackage{slashed}
\usepackage{braket}
\usepackage{enumitem}
\usepackage{bbm}
\usepackage[svgnames]{xcolor}
\usepackage[colorlinks,citecolor=DarkGreen,linkcolor=FireBrick,urlcolor=FireBrick,linktocpage=true,unicode,hypertexnames=false]{hyperref}
\usepackage{multirow}
\usepackage{graphicx}
\graphicspath{{figures/}}
\makeatletter
\@ifundefined{PRISM@MACROS@LOADED}{\gdef\PRISM@MACROS@LOADED{}}{ }
\makeatother

\providecommand{\PaperTitle}{D-DOIT: Training-free Adaptation of Discrete Diffusion via Doob's \textnormal{\itshape h}-Transform}

\providecommand{\StateSymbol}{\ensuremath{\bar{\bX}_{t_l}}}

\providecommand{\MainResult}{\ensuremath{6.19}}

\providecommand{\MainResultUncertainty}{\ensuremath{0.05}}

\usepackage{graphicx}
\usepackage{tikz}
\usepackage{tikz-cd}
\usepackage{bm}
\usepackage{physics}
\usepackage{xcolor}
\usepackage{natbib}
\usepackage{mdframed}
\usepackage{nicefrac}
\usepackage{booktabs}
\usepackage{multirow}
\usepackage{lipsum}
\usepackage{wrapfig,lipsum,booktabs}
\usepackage{placeins}
\usepackage{needspace}
\usepackage{blindtext}
\usepackage[ruled]{algorithm}
\usepackage{algorithmic}

\usepackage{titletoc}
\usepackage{todonotes}
\usepackage{titletoc}
\usepackage{setspace}
\usepackage{dsfont}
\usepackage{adjustbox}
\usepackage{longtable, lscape}
\usepackage{listings}

\definecolor{codegreen}{rgb}{0,0.6,0}
\definecolor{codegray}{rgb}{0.5,0.5,0.5}
\definecolor{codepurple}{rgb}{0.58,0,0.82}
\definecolor{backcolour}{rgb}{0.97,0.97,0.97}
\lstdefinestyle{pytorchstyle}{
  backgroundcolor=\color{backcolour},
  commentstyle=\color{codegreen}\itshape,
  keywordstyle=\color{magenta}\bfseries,
  numberstyle=\tiny\color{codegray},
  stringstyle=\color{codepurple},
  basicstyle=\ttfamily\footnotesize,
  breakatwhitespace=false,
  breaklines=true,
  captionpos=b,
  keepspaces=true,
  numbers=left,
  numbersep=5pt,
  showspaces=false,
  showstringspaces=false,
  showtabs=false,
  tabsize=4,
  xleftmargin=10pt,
  xrightmargin=10pt,
  frame=single,
  rulecolor=\color{black!20}
}
\usepackage{xspace}
\usepackage[capitalize,noabbrev]{cleveref}
\usepackage{CJK}
\usepackage{array}
\usepackage{makecell}

\newcommand{\bea}{
\begin{eqnarray}}
  \newcommand{\eea}{
\end{eqnarray}}

\newcommand{\lG}{\lambda}

\usepackage{slashed}

\def\({\left(}
\def\){\right)}
\def\[{\left[}
\def\]{\right]}

\usepackage{dashbox}
\usepackage{xcolor}
\usepackage{colortbl}
\definecolor{lightyellow}{rgb}{1.0, 0.95, 0.7}
\definecolor{Blue}{rgb}{0, 0, 0.8}
\definecolor{blue}{rgb}{0,0,1}
\definecolor{darkgreen}{rgb}{0,0.40,0}
\definecolor{firebrick}{rgb}{0.698,0.133,0.133}

\definecolor{colorA}{rgb}{1,0,0}
\definecolor{colorB}{rgb}{0,0.3,1}
\definecolor{colorC}{rgb}{0.9,0.8,0.2}
\definecolor{colorD}{rgb}{0,0.65,0}
\definecolor{lesslightgray}{rgb}{0.5,0.5,0.5}
\definecolor{light-gray}{gray}{0.95}

\newcommand{\calC}{\mathcal{C}}

\newcommand{\calG}{\mathcal{G}}

\newcommand{\calV}{\mathcal{V}}

\newcommand{\calX}{\mathcal{X}}

\newcommand{\bW}{\mathbf{W}}
\newcommand{\bX}{\mathbf{X}}

\newcommand{\bx}{\mathbf{x}}
\newcommand{\by}{\mathbf{y}}

\newcommand{\argmax}{\mathop{\mathrm{argmax}}}

\newcommand{\NFE}{\mathrm{NFE}}

\newcommand{\spl}{\mathrm{split}}
\newcommand{\eend}{\mathrm{end}}

\def\R{\mathbb{R}}

\def\E{\mathbb{E}}

\let\cite\citep
\usepackage{amsthm}
\usepackage[most]{tcolorbox}

\makeatletter
\def\th@remark{%
  \thm@headfont{\bfseries}%
  \normalfont % body font
  \thm@preskip\topsep \divide\thm@preskip\tw@
  \thm@postskip\thm@preskip
}
\makeatother
\theoremstyle{definition}
\newtheorem{theorem}{Theorem}[section]
\tcolorboxenvironment{theorem}{
  breakable,
  colback=black!10,
  colframe=white,%
  width=\dimexpr\linewidth+10pt\relax,%
  enlarge left by=-5pt,%
  enlarge right by=-5pt,%
  boxsep=5pt,%
  boxrule=0pt,
  left=0pt,right=0pt,top=0pt,bottom=0pt,
  sharp corners,
  before skip=\topsep,
  after skip=\topsep
}

\tcolorboxenvironment{lemma}{
  breakable,
  colback=black!10,
  colframe=white,%
  width=\dimexpr\linewidth+10pt\relax,%
  enlarge left by=-5pt,%
  enlarge right by=-5pt,%
  boxsep=5pt,%
  boxrule=0pt,
  left=0pt,right=0pt,top=0pt,bottom=0pt,
  sharp corners,
  before skip=\topsep,
  after skip=\topsep
}

\tcolorboxenvironment{corollary}{
  breakable,
  colback=black!10,
  colframe=white,%
  width=\dimexpr\linewidth+10pt\relax,%
  enlarge left by=-5pt,%
  enlarge right by=-5pt,%
  boxsep=5pt,%
  boxrule=0pt,
  left=0pt,right=0pt,top=0pt,bottom=0pt,
  sharp corners,
  before skip=\topsep,
  after skip=\topsep
}
\newtheorem{proposition}{Proposition}[section]
\tcolorboxenvironment{proposition}{
  breakable,
  colback=black!10,
  colframe=white,%
  width=\dimexpr\linewidth+10pt\relax,%
  enlarge left by=-5pt,%
  enlarge right by=-5pt,%
  boxsep=5pt,%
  boxrule=0pt,
  left=0pt,right=0pt,top=0pt,bottom=0pt,
  sharp corners,
  before skip=\topsep,
  after skip=\topsep
}
\theoremstyle{definition}
\newtheorem{definition}{Definition}[section]
\tcolorboxenvironment{definition}{
  breakable,
  colback=black!10,
  colframe=white,%
  width=\dimexpr\linewidth+10pt\relax,%
  enlarge left by=-5pt,%
  enlarge right by=-5pt,%
  boxsep=5pt,%
  boxrule=0pt,
  left=0pt,right=0pt,top=0pt,bottom=0pt,
  sharp corners,
  before skip=\topsep,
  after skip=\topsep
}
\theoremstyle{remark}
\newtheorem{remark}{Remark}[section]

\tcolorboxenvironment{assumption}{
  breakable,
  colback=black!10,
  colframe=white,%
  width=\dimexpr\linewidth+10pt\relax,%
  enlarge left by=-5pt,%
  enlarge right by=-5pt,%
  boxsep=5pt,%
  boxrule=0pt,
  left=0pt,right=0pt,top=0pt,bottom=0pt,
  sharp corners,
  before skip=\topsep,
  after skip=\topsep
}

\tcolorboxenvironment{problem}{
  breakable,
  colback=black!10,
  colframe=white,%
  width=\dimexpr\linewidth+10pt\relax,%
  enlarge left by=-5pt,%
  enlarge right by=-5pt,%
  boxsep=5pt,%
  boxrule=0pt,
  left=0pt,right=0pt,top=0pt,bottom=0pt,
  sharp corners,
  before skip=\topsep,
  after skip=\topsep
}

\crefname{theorem}{Theorem}{Theorems}
\crefname{proposition}{Proposition}{Propositions}
\crefname{lemma}{Lemma}{Lemmas}
\crefname{corollary}{Corollary}{Corollaries}
\crefname{definition}{Definition}{Definitions}
\crefname{assumption}{Assumption}{Assumptions}
\crefname{remark}{Remark}{Remarks}
\crefname{problem}{Problem}{Problems}
\crefname{property}{Property}{property}

\numberwithin{equation}{section}
\numberwithin{theorem}{section}
\numberwithin{proposition}{section}
\numberwithin{definition}{section}
\numberwithin{lemma}{section}
\numberwithin{assumption}{section}
\numberwithin{remark}{section}

\usepackage{lipsum}

\makeatletter
\@ifundefined{widebar}{%
  \let\save@mathaccent\mathaccent
  \newcommand*\if@single[3]{%
    \setbox0\hbox{${\mathaccent"0362{#1}}^H$}%
    \setbox2\hbox{${\mathaccent"0362{\kern0pt#1}}^H$}%
    \ifdim\ht0=\ht2 #3\else #2\fi
  }
  \newcommand*\rel@kern[1]{\kern#1\dimexpr\macc@kerna}
  \newcommand*\widebar[1]{\@ifnextchar^{{\wide@bar{#1}{0}}}{\wide@bar{#1}{1}}}
  \newcommand*\wide@bar[2]{\if@single{#1}{\wide@bar@{#1}{#2}{1}}{\wide@bar@{#1}{#2}{2}}}
  \newcommand*\wide@bar@[3]{%
    \begingroup
    \def\mathaccent##1##2{%
      \let\mathaccent\save@mathaccent
      \if#32 \let\macc@nucleus\first@char \fi
      \setbox\z@\hbox{$\macc@style{\macc@nucleus}_{}$}%
      \setbox\tw@\hbox{$\macc@style{\macc@nucleus}{}_{}$}%
      \dimen@\wd\tw@
      \advance\dimen@-\wd\z@
      \divide\dimen@ 3
      \@tempdima\wd\tw@
      \advance\@tempdima-\scriptspace
      \divide\@tempdima 10
      \advance\dimen@-\@tempdima
      \ifdim\dimen@>\z@ \dimen@0pt\fi
      \rel@kern{0.6}\kern-\dimen@
      \if#31
      \overline{\rel@kern{-0.6}\kern\dimen@\macc@nucleus\rel@kern{0.4}\kern\dimen@}%
      \advance\dimen@0.4\dimexpr\macc@kerna
      \let\final@kern#2%
      \ifdim\dimen@<\z@ \let\final@kern1\fi
      \if\final@kern1 \kern-\dimen@\fi
      \else
      \overline{\rel@kern{-0.6}\kern\dimen@#1}%
      \fi
    }%
    \macc@depth\@ne
    \let\math@bgroup\@empty \let\math@egroup\macc@set@skewchar
    \mathsurround\z@ \frozen@everymath{\mathgroup\macc@group\relax}%
    \macc@set@skewchar\relax
    \let\mathaccentV\macc@nested@a
    \if#31
    \macc@nested@a\relax111{#1}%
    \else
    \def\gobble@till@marker##1\endmarker{}%
    \futurelet\first@char\gobble@till@marker#1\endmarker
    \ifcat\noexpand\first@char A\else
    \def\first@char{}%
    \fi
    \macc@nested@a\relax111{\first@char}%
    \fi
    \endgroup
  }
}{}
\makeatother

\makeatletter
\newcommand*{\redefinesymbolwitharg}[1]{%
  \expandafter\let\csname ltx#1\expandafter\endcsname\csname #1\endcsname
  \@namedef{#1}{\@ifnextchar{^}{\@nameuse{#1@}}{\@nameuse{#1@}^{}}}%
  \expandafter\def\csname #1@\endcsname^##1##2{%
    \csname ltx#1\endcsname\ifx!##1!\else^{##1}\fi\mathopen{}\mathclose\bgroup\left(##2\aftergroup\egroup\right)
  }%
}
\makeatother
\redefinesymbolwitharg{sin}
\redefinesymbolwitharg{cos}
\title{\PaperTitle}
\author{%
  Jieke Wu\thanks{Equal contribution.}$^{\:\;\ddag}$, \quad
    Qijie Zhu$^{*\S}$, \quad
    Weimin Wu$^{*\dagger}$, \quad
    Zeqi Ye$^{\natural}$, \quad
    Minshuo Chen$^{\natural}$, \quad
    Han Liu$^{\dagger \S}$
  \\ \vspace{0.5em}
  {\small
  \begin{tabular}{@{}c@{}}
    $^\dagger\;$Center for Foundation Models and Generative AI, Northwestern University\\
    Department of Computer Science, Northwestern University\\
    $^\ddag\;$Department of Computer Science, King Abdullah University of Science and Technology\\
    $^\S\;$Department of Statistics and Data Science, Northwestern University\\
    $^\natural\;$Industrial Engineering \& Management Sciences, Northwestern University
  \end{tabular}
  }\\
  {\small
  \begin{tabular}{@{}c@{}}
    \texttt{\href{mailto:jieke.wu@kaust.edu.sa}{jieke.wu@kaust.edu.sa}},
    \texttt{\{\href{mailto:qijiezhu2029@u.northwestern.edu}{qijiezhu2029},\href{mailto:wwm@u.northwestern.edu}{wwm},\href{mailto:zeqiye2029@u.northwestern.edu}{zeqiye2029}\}@u.northwestern.edu}\\
    \texttt{\href{mailto:minshuo.chen@northwestern.edu}{minshuo.chen@northwestern.edu}}, \texttt{\href{mailto:hanliu@northwestern.edu}{hanliu@northwestern.edu}}
  \end{tabular}
  } \\
}

\def\Snospace~{\S{}}

\newcommand{\sref}[2]{\hyperref[#2]{#1 \ref{#2}}}

\begin{document}
% Give tables and figures enough air for the camera-ready layout.
\setlength{\textfloatsep}{11pt plus 2pt minus 3pt}
\setlength{\floatsep}{9pt plus 2pt minus 2pt}
\setlength{\intextsep}{9pt plus 2pt minus 2pt}
\setlength{\abovecaptionskip}{6pt}
\setlength{\belowcaptionskip}{6pt}

\maketitle

\begin{abstract}
% Reward-guided adaptation is essential for applying discrete diffusion models to
% sequence design, but existing methods either update model parameters, rely on
% differentiable rewards, or incur substantial inference cost through large
% particle populations or search procedures.
We propose D-DOIT (Discrete Doob-Oriented Inference-time Transformation), a training-free and efficient adaptation method for discrete diffusion models with generic rewards.
D-DOIT formulates adaptation as sampling from a reward-tilted target distribution and realizes this transport through Doob's $h$-transform of the discrete diffusion reverse kernel, using only reward values rather than reward gradients.
Unlike continuous diffusion, masked discrete diffusion samples categorical token-reveal transitions rather than continuous state updates.
D-DOIT derives the corresponding discrete Doob's $h$-transform, which guides sampling by reweighting reverse transition probabilities instead of adding a drift correction.
To make this transformation practical, D-DOIT avoids expensive future rollouts.
At each guided step, D-DOIT samples candidate next states, uses the model prediction head to complete each candidate into a clean sequence, evaluates each completion with the reward oracle, and resamples the next state with probabilities proportional to the rewards.
An optional late-stage best-of-$K$ refinement further improves sample quality by branching trajectories only near the end of denoising, avoiding the $K$-fold cost over the full trajectory.
Empirically, across regulatory DNA design and protein inverse folding benchmarks, D-DOIT outperforms training-free guidance baselines.
It improves enhancer activity and cell-type specificity while preserving sequence naturalness, and achieves the highest success rate in protein inverse folding.

{\bf Project Page:}
\href{https://github.com/WeiminWu2000/D-DOIT}{https://github.com/WeiminWu2000/D-DOIT}.

\end{abstract}

\section{Introduction}
\label{sec:intro}
We propose D-DOIT (\textbf{D}iscrete \textbf{D}oob-\textbf{O}riented
\textbf{I}nference-time \textbf{T}ransformation), a training-free sampler for
reward-guided masked discrete diffusion. Discrete diffusion has become a strong
generative modeling paradigm for finite-vocabulary data
\citep{austin2021structured,campbell2022continuous}. Masked discrete diffusion
models are especially attractive for biological sequences because they predict
clean tokens from partially masked states and provide tractable categorical
reverse kernels \citep{sahoo2024simple,shi2024simplified,lou2024sedd}. These
models have recently been used for DNA and protein generation
\citep{stark2024dirichlet,sarkar2024designing,wang2025drakes}, where the
scientific objective is rarely just to sample from the pre-trained distribution.
Instead, a practitioner usually wants sequences with high activity, specificity,
or other task-level scores for sequence design.

Reward adaptation creates a tension between performance, generality, and cost.
Training-based approaches can improve reward optimization, but they update the
backbone and often require reward gradients, extra training, and hyperparameter
tuning \citep{wang2025drakes,han2025trajectory}. Training-free
methods keep the backbone fixed, but many either assume differentiable rewards
or rely on particle and search procedures whose cost grows with trajectory
history or search width
\citep{nisonoff2024unlocking,wu2024practical,li2025derivativefree,uehara2025reward,schiff2024simple,rectorbrooks2024ddpp,frey2024plugplay,singhal2025general}.
This motivates the question studied here: can a frozen masked diffusion model be
adapted to generic black-box rewards by modifying only its inference-time
sampling law?

Our starting point is to view adaptation as a distributional transport problem.
Let $p_\theta$ be the clean-sequence distribution induced by the frozen masked
diffusion reverse chain, and let $R:\calX\to\mathbb R$ be a sequence-level
reward. Rather than optimizing reward in isolation, we target the reward-tilted
distribution $\pi_\beta(\bx)\propto p_\theta(\bx)\exp(R(\bx)/\beta)$, which
keeps the pre-trained model as a sequence prior while favoring high-reward
outputs. Doob's $h$-transform gives the exact local transformation needed to
sample from this tilted law. The central object is the downstream value
$h(\bar\bx_{t_l},t_l)$, the expected terminal reward weight under the frozen
reverse chain conditioned on the current partially denoised state.

The discrete setting changes how the transform must be implemented. In
continuous diffusion, Doob guidance can be expressed as an additive score or
drift correction \citep{nguyen2025h,zhu2026trainingfreeadaptationdiffusionmodels,heng2021simulating,denker2024deft}. In masked discrete diffusion, the
state space is categorical, generation proceeds through irreversible
token-reveal transitions, and there is no Euclidean score field. D-DOIT
therefore applies Doob's $h$-transform at the kernel level:
each frozen reverse transition is multiplicatively reweighted by the downstream
$h$-value of the candidate next state. This yields a discrete analogue of the
continuous Doob correction while preserving the frozen parameters and allowing
black-box reward queries.

The exact transform still requires the future reward-to-go for every candidate
reverse state, which is intractable by repeated rollouts. D-DOIT
makes the transform executable with two approximations. First, it estimates
candidate $h$-values by using the masked diffusion prediction head to complete
each candidate into a clean sequence, then evaluating the reward oracle on
that completion. Second, it draws a finite set of candidate reverse states and
selects among them with Doob-weighted probabilities. An optional best-of-$K$
refinement branches only after a chosen split update, which hedges irreversible
token reveals without multiplying the full denoising trajectory cost. The same
sampler gives two operating points: D-DOIT-LBOK-S for a lower-budget regime and
D-DOIT-Q for high-budget quality evidence.

% Because these components replace the
% true rollout value and the exact kernel, their temperatures, sharpening,
% guidance window, and terminal selection define a practical operating point.
% This operating point does not inherit the exact-target guarantee of the
% reference transform.

We evaluate D-DOIT on regulatory DNA sequence design and protein inverse
folding. On the DRAKES HepG2 enhancer benchmark, the near-TDS-speed
D-DOIT-LBOK-S matches TDS on predicted activity while substantially improving
accessibility, 3-mer correlation, and motif correlation. The high-budget
D-DOIT-Q setting reaches $\MainResult$ predicted activity and exceeds training-based
DRAKES on Pred-Activity, 3-mer Corr, and JASPAR Corr, while remaining slightly
lower on ATAC-Acc. At a similar wall-clock budget, D-DOIT-Q also improves
activity, accessibility, and 3-mer correlation over soft value-based decoding
(SVDD-PM) and wall-clock-matched endpoint search. Endpoint search
achieves higher JASPAR correlation. On the Ctrl-DNA cell-type-specific enhancer
benchmark, D-DOIT is a clear training-free method that moves
the specificity reward. D-DOIT-Q raises $\Delta R$ to $0.119$, with a trade-off
in motif preservation, and D-DOIT-LBOK-S provides a lower-budget positive-reward
point. At a matched wall-clock budget, however, endpoint search captures much of
this gain, so the margin is smaller than on HepG2.
Guidance with alternative predictors and transfer to held-out evaluators retain
the gains over unguided sampling. On protein inverse folding, the same frozen
backbone reaches $68.1\%$ joint stability--structure success. These results
support D-DOIT as a training-free mechanism for reward-guided sequence
generation, not as a matched-cost replacement for every high-budget or
training-based method.

In summary, we make three contributions:
\begin{itemize}[leftmargin=1em, topsep=0.3em, itemsep=0.5em, parsep=0pt, partopsep=0pt]
  \item \textbf{Discrete Doob $h$-transform for masked diffusion.}
    We formulate reward adaptation as sampling from a reward-tilted target
    distribution and derive the corresponding Doob $h$-transform of the masked
    discrete diffusion reverse kernel. The resulting guidance reweights
    categorical token-reveal transitions, in contrast to the additive drift or
    score correction used in continuous diffusion. The exact reference transform
    is kept separate from the executable approximations built on it.
  \item \textbf{Executable finite-candidate sampler.}
    We make the exact transform practical with a prediction-head $h$-estimator,
    finite-candidate Doob-weighted selection, and an optional best-of-$K$
    refinement for irreversible reveal decisions. The method uses only reward
    values and does not update the diffusion backbone. We characterize how
    tempering, sharpening, and terminal selection deviate from the exact target.
  \item \textbf{Biological validation and trade-offs.}
    We test the same frozen-backbone sampler on HepG2 enhancer design,
    cell-type-specific enhancer design, and protein inverse folding. D-DOIT
    improves over training-free baselines across the reported reward-guided
    design tasks, while preserving the relevant trade-offs in naturalness, motif
    structure, and inference budget.
    % Nearest-baseline and oracle-transfer
    % controls support the reported gains without implying universal dominance.
\end{itemize}

\paragraph{Organization.}
\Cref{sec:prelim} reviews the masked discrete diffusion backbone, defines the
reward-tilted target, and states the exact discrete Doob $h$-transform.
\Cref{sec:method} develops the practical D-DOIT sampler, including the
prediction-head $h$-estimator, finite-candidate weighted update, and optional
best-of-$K$ refinement. \Cref{sec:exp} reports experiments on regulatory DNA
design, cell-type-specific enhancer design, and protein inverse folding. The
appendix provides related work (\cref{app:related}), the proof of the discrete
Doob $h$-transform (\cref{app:proof-discrete-doob-htransform}), the standard masked
diffusion sampler (\cref{sec:mdlm-baseline}), and full experimental details
(\cref{app:exp:drakes,app:exp:ctrldna,app:exp:protein}).

\paragraph{Notation.}
We write $\calV=\{A,G,T,C\}$ for the clean token alphabet, $d=|\calV|$, and
$\calX=\calV^N$ for length-$N$ DNA sequences. Let $\mathsf m$ be the mask
token, $\widetilde{\calV}=\calV\cup\{\mathsf m\}$, and
$\widetilde{\calX}=\widetilde{\calV}^N$. We use $\bX_{t_l}$ for the forward
process and $\StateSymbol$ for the learned reverse chain, which runs from the
all-\textsc{Mask} state $\bar\bx_{t_L}$ to a clean sequence $\bar\bx_0$. We use
$q$ for the forward corruption kernel and analytic posterior, $P_{\theta,t_l}$
for the learned reverse kernel, $p_{\theta,t_l}$ for the reverse-chain marginal,
and $R:\calX\to\R$ with $\phi(\bx_0)=\exp(R(\bx_0)/\beta)$ for the black-box
reward and reward weight. For the sampler, $M$ denotes the local candidate
budget, $n_h$ the number of completions per candidate, $\tau$ and $\gamma$
control proposal and completion tempering, $\omega$ sharpens value-based
selection, and $K$ denotes the best-of-$K$ branching width.

\section{Preliminaries}
\label{sec:prelim}

We first define the frozen masked diffusion reverse chain that serves as the
base law for D-DOIT. We then state the reward-tilted target distribution and
the discrete Doob $h$-transform that transports this base law to the target
without changing the model parameters \citep{rogers2000diffusions,sarkka2019applied}.
We keep the exact reference transform separate from the executable
approximation built on it, which \cref{sec:method} develops.

\subsection{Masked Discrete Diffusion Models}
\label{sec:background-masked-discrete-diffusion}

We specialize to masked discrete diffusion. Let $\calV$ be the clean token
alphabet, $d=|\calV|$, and $\mathsf m$ a special mask token. Define
$\widetilde{\calV}:=\calV\cup\{\mathsf m\}$, with sequence spaces
$\calX=\calV^N$ and $\widetilde{\calX}=\widetilde{\calV}^N$. The data
distribution is supported on $\calX$, while the forward process evolves on
$\widetilde{\calX}$ by progressively replacing clean tokens with $\mathsf m$.
On the grid $t_l=lT/L$, let $b_l\in[0,1]$ be the masking probability on
$[t_{l-1},t_l]$. The single-token forward kernel is
\begin{align*}
  q_{t_{l-1},t_l}(y\mid x)
  =
  (1-b_l)\mathbf 1\{y=x\}
  +
  b_l\mathbf 1\{y=\mathsf m\},
  \quad x,y\in\widetilde{\calV},
\end{align*}
so with probability $1-b_l$ the token stays fixed and otherwise it jumps to the
mask state; once masked it stays masked. With
$\alpha_l=\prod_{r=1}^{l}(1-b_r)$, a clean token $x_0\in\calV$ satisfies
\begin{align*}
  q_{0,t_l}(y\mid x_0)
  =\mathbb{P}(X_{t_l}=y\mid X_0=x_0)
  =\alpha_l\mathbf 1\{y=x_0\}
  +
  (1-\alpha_l)\mathbf 1\{y=\mathsf m\}.
\end{align*}
We choose $\alpha_L\approx 0$ for a sufficiently large $T$, so that $q_{0,T}$
is concentrated near the mask state. The forward masking process is applied
independently across coordinates. The conditional one-step time reversal also
has a closed form, and \cref{app:masked-diffusion-details} records the forward
transition matrices and the analytic reverse kernel that the learned sampler
approximates.

Since $\bx_0$ is unknown during generation, masked diffusion models
approximate the clean-token posterior with a trained neural network.
For each coordinate $j$, let
$\mu_\theta^j(\cdot\mid \bx_{t_l},t_l)\in \Delta(\calV)$ denote the neural
network's approximation to
$\mathbb P(X_0^j=\cdot \mid \mathbf X_{t_l}=\bx_{t_l})$.
Plugging this approximation into the analytic reverse posterior gives the
learned one-step reverse transition from $t_l$ to $t_{l-1}$:
\begin{align}
  P_{\theta,t_l}(x_{t_{l-1}}^j=y\mid \bx_{t_l})
  =
  \begin{cases}
    (1-\alpha_{l-1})/(1-\alpha_l),
    & x_{t_l}^j=\mathsf m,\ y=\mathsf m,\\
    (\alpha_{l-1}-\alpha_l)/(1-\alpha_l)\,
    \mu_\theta^j(y\mid \bx_{t_l},t_l),
    & x_{t_l}^j=\mathsf m,\ y\in\calV,\\
    \mathbf 1\{y=x_{t_l}^j\},
    & x_{t_l}^j\neq\mathsf m.
  \end{cases}
  \label{eq:mdlm-transition}
\end{align}
Finally, the learned reverse model factorizes over coordinates,
\begin{align*}
  P_{\theta,t_l}(\bx_{t_{l-1}}\mid \bx_{t_l})
  &=
  \prod_{j=1}^N
  P_{\theta,t_l}(x_{t_{l-1}}^j\mid \bx_{t_l}).
\end{align*}
Starting from an initial masked state $\bar{\bX}_{t_L}$, we iteratively sample
$\bar{\bX}_{t_{l-1}}\sim P_{\theta,t_l}(\cdot\mid \StateSymbol)$ from
$l=L$ to $l=1$ and obtain the final clean sample $\bar{\bX}_0$.
Let $p_{\theta,t_l}$ denote the marginal law of $\StateSymbol$ under this
reverse chain, and write $p_\theta=p_{\theta,0}$ for the clean-time law.
This learned reverse chain is the frozen base sampler. D-DOIT leaves its
parameters fixed and modifies only its reverse transition probabilities at
inference time. We next define the target these modified transitions should
realize.

\subsection{Reward-tilted Target and Doob's \texorpdfstring{$h$}{h}-transform}
\label{sec:htransform-discrete}

We now formalize the inference-time adaptation objective. Given a reward
function $R:\calX\to\mathbb R$ and a temperature parameter $\beta>0$, define
the positive reward weight
\begin{align}
  \label{eq:reward-weight}
  \phi(\bx_0)=\exp(R(\bx_0)/\beta).
\end{align}
The desired clean-time target is the reward-tilted version of the frozen model
law,
\begin{align}
  \label{eq:tilted-target}
  \pi_\beta(\bx_0)
  &=
  \frac{
    p_\theta(\bx_0)\phi(\bx_0)
  }{
    Z_\beta
  },\quad
  Z_\beta=
  \E_{\bX_0\sim p_\theta}
  \left[
    \phi(\bX_0)
  \right].
\end{align}
Thus adaptation is a distributional problem: change the terminal law from
$p_\theta$ to $\pi_\beta$ while preserving the frozen reverse chain as the base
proposal. The discrete Doob $h$-function is the object that makes this
transport local in time and gives a kernel-level correction.

\begin{definition}[Discrete Doob $h$-function]
  \label{def:discrete-doob-h-function}
  For the frozen reverse chain defined above, the discrete Doob $h$-function is
  the downstream reward weight conditioned on a partially denoised state:
  \begin{align}
    \label{eq:prelim-discrete-doob-h-function}
    h(\bar \bx_{t_l},t_l)
    =
    \E_{p_\theta}
    \left[
      \exp(R(\bX_0)/\beta)
      \mid
      \StateSymbol=\bar \bx_{t_l}
    \right],
    \qquad l=0,\ldots,L.
  \end{align}
  At clean time, this reduces to $h(\bx_0,0)=\phi(\bx_0)$. At an intermediate
  time, $h(\bar\bx_{t_l},t_l)$ is the reward-to-go under the frozen reverse
  chain. It therefore measures how promising the current partial sequence is
  for the terminal reward objective at clean time under the frozen process.
\end{definition}

The next theorem states the exact kernel transform induced by this value
function. It is the discrete analogue of the continuous Doob correction, but it
acts by reweighting transition probabilities rather than by adding a vector
field to a continuous score.

\begin{theorem}[Discrete Doob's $h$-transform]
  \label{thm:discrete-doob-h-transform}
  In the fully masked terminal setting, $\bar \bX_{t_L}$ is deterministic, so
  $h(\bar \bX_{t_L},t_L)=Z_\beta$ and the terminal distribution is unchanged.
  Let $h$ be defined as in \cref{def:discrete-doob-h-function}. Then Doob's
  $h$-transform induces the reverse kernel
  \begin{align}
    \label{eq:discrete-doob-kernel}
    P_{\theta,t_l}^h
    (
      \bar \bx_{t_{l-1}}
      \mid
      \bar \bx_{t_l}
    )
    =
    P_{\theta,t_l}
    (
      \bar \bx_{t_{l-1}}
      \mid
      \bar \bx_{t_l}
    )
    \frac{
      h(\bar \bx_{t_{l-1}},t_{l-1})
    }{
      h(\bar \bx_{t_l},t_l)
    },
    \qquad l=1,\ldots,L.
  \end{align}
  The resulting Doob-transformed reverse chain has marginal distribution
  \begin{align*}
    p_{\theta,t_l}^h(\bar \bx_{t_l})
    =
    \frac{
      p_{\theta,t_l}(\bar \bx_{t_l})
      h(\bar \bx_{t_l},t_l)
    }{
      Z_\beta
    },
    \qquad l=0,\ldots,L,
  \end{align*}
  In particular, since
  \begin{align*}
    h(\bx_0,0)
    =
    \exp(R(\bx_0)/\beta),
  \end{align*}
  the clean-time marginal of the Doob-transformed reverse chain is exactly the
  reward-tilted distribution
  \begin{align*}
    p_{\theta,0}^h(\bar\bx_0)
    =
    \pi_\beta(\bar\bx_0)
    =
    \frac{
      p_{\theta,0}(\bar\bx_0)\exp(R(\bar\bx_0)/\beta)
    }{
      Z_\beta
    }.
  \end{align*}
\end{theorem}

The proof is deferred to \cref{app:proof-discrete-doob-htransform}. The theorem
identifies the exact population object that D-DOIT approximates. If the
downstream value ratio
$h(\bar \bx_{t_{l-1}},t_{l-1})/h(\bar \bx_{t_l},t_l)$ were available at every
reverse step, then local reweighting of the frozen transition kernel would
sample exactly from $\pi_\beta$ at clean time. This guarantee is a statement
about the true rollout value and the exact normalized kernel, not the
prediction-head implementation. In \cref{sec:method}, the head proxy, finite
candidate set, proposal and completion temperatures, weight sharpening, and
terminal selection each have a distinct role. Finite-candidate convergence
alone cannot remove a systematic value or proposal mismatch. We keep the
original $P$-referenced $\pi_\beta$ as the theoretical benchmark;
\cref{app:tempered-reference} separately states the exact transform for a
temperature-modified reference chain. The practical method in
\cref{sec:method} estimates this ratio without retraining the masked diffusion model. The continuous analogue appears as an additive score
correction $\nabla\log p_t+\nabla\log h_t$. The categorical and
irreversible reverse dynamics of masked diffusion have no Euclidean score field,
and instead require the multiplicative kernel reweighting in
\eqref{eq:discrete-doob-kernel}; \cref{app:continuous-doob} records the
continuous derivation.

\section{D-DOIT: A Discrete Sampler Based on Doob's \texorpdfstring{$h$}{h}-transform}
\label{sec:method}
In this section, we introduce \textbf{D-DOIT}, a training-free sampler that
guides a frozen masked discrete diffusion model toward the reward-tilted target
in \eqref{eq:tilted-target} using black-box rewards $R:\calX\to\mathbb R$
defined on clean sequences. The exact transform in
\cref{thm:discrete-doob-h-transform} says what the sampler should do:
reweight each frozen reverse kernel by the downstream value ratio of the
discrete Doob $h$-function. The challenge is computational, since evaluating
this value exactly would require future denoising rollouts from every candidate
next state. D-DOIT therefore follows an exact-object-to-estimator design:
it estimates the downstream value with the frozen prediction head, approximates
the transformed kernel with finite candidates, and optionally branches late to
hedge irreversible token reveals. \Cref{fig:overview} summarizes the construction.

% BEGIN INLINE figures/fig_ddoit_overview
\begin{figure}[!t]
  \centering
  \includegraphics[width=\textwidth]{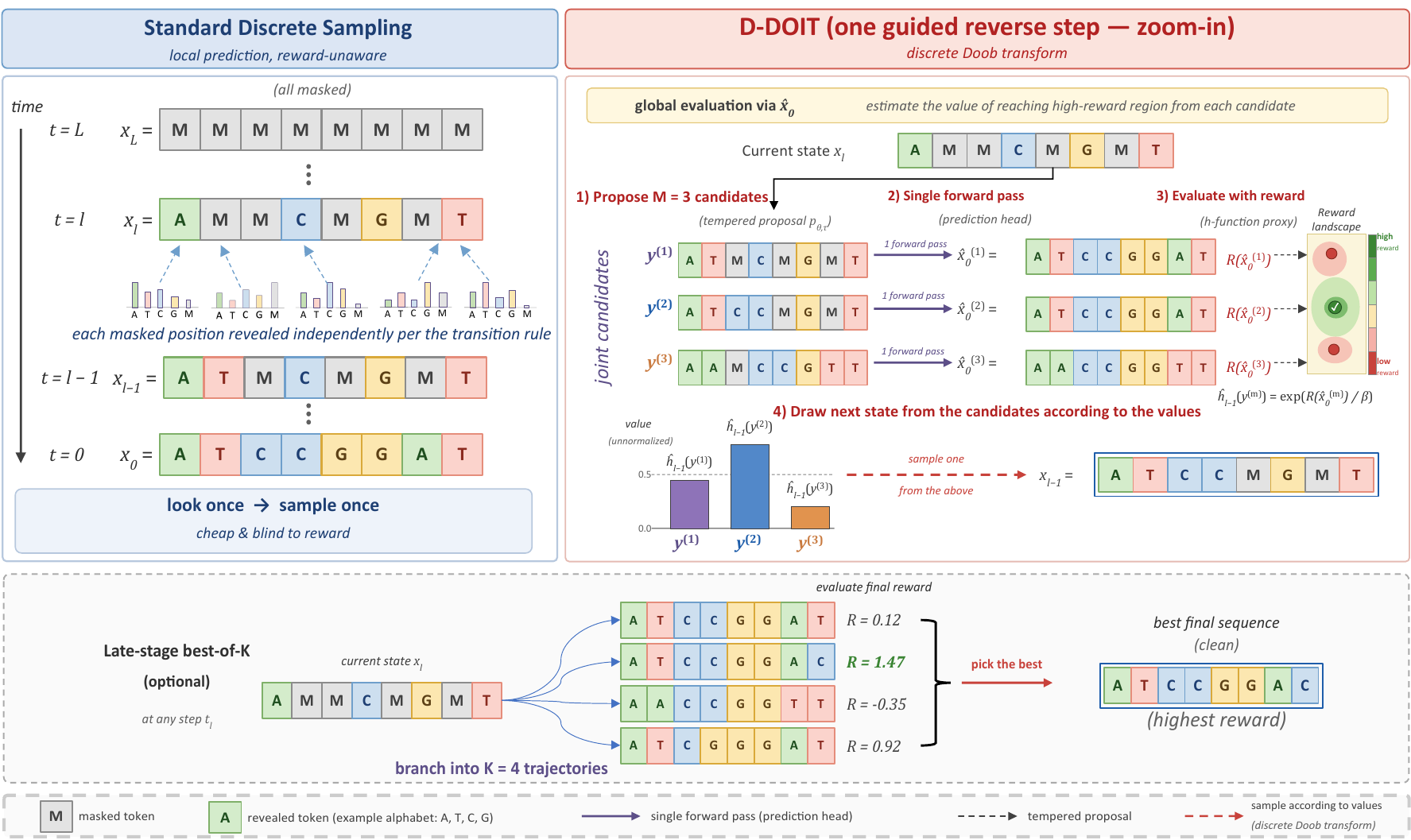}
  \caption{\textbf{Overview of D-DOIT.}
    \textbf{Left:} Standard masked diffusion unmasks positions independently at each reverse step~\eqref{eq:mdlm-transition}; it is cheap but reward-unaware.
    \textbf{Right:} One guided step approximately tilts the reverse kernel: it draws $M{=}3$ candidates from a tempered proposal~\eqref{eq:tempered-proposal}, completes and scores them with the prediction head and reward to obtain proxy $h$-values, and selects the next state with sharpened $h$-weights~\eqref{eq:tilted-resample}.
    \textbf{Bottom:} Optional best-of-$K$ refinement duplicates each trajectory into $K{=}4$ copies after the split and keeps the highest-reward suffix~\eqref{eq:bok-select}, sharing prefix computation. Revealed tokens stay fixed; the practical sampler does not inherit the exact-target guarantee.}
  \label{fig:overview}
  \vspace{-1em}
\end{figure}
% END INLINE figures/fig_ddoit_overview

\subsection{D-DOIT Algorithm}
\label{sec:doobmdm}
The transformed kernel in \eqref{eq:discrete-doob-kernel} depends on the
downstream value $h(\bar \by_{t_{l-1}},t_{l-1})$ for each candidate reverse
state $\bar \by_{t_{l-1}}\in\widetilde{\calX}$ reachable from the current
state $\bar \bx_{t_l}$. This subsection turns that exact identity into an
executable sampler. The first component estimates each candidate's downstream
value with a prediction-head completion. The second component uses these values
to form a finite-candidate approximation to the transformed kernel. The final
component adds an optional best-of-$K$ refinement for the irreversible reveal
structure of masked diffusion.

\paragraph{Prediction-head $h$-estimator.}
The exact value $h(\bar \by_{t_{l-1}},t_{l-1})$ averages the reward weight over all complete denoising continuations from the candidate next state $\bar \by_{t_{l-1}}$.
Naive Monte Carlo would roll out the frozen reverse chain many times from every candidate. D-DOIT avoids these nested rollouts by evaluating candidates with the masked discrete diffusion prediction head at $\bar \by_{t_{l-1}}$. Recall from \cref{sec:background-masked-discrete-diffusion} that
$\mu_\theta^j(\cdot\mid \bx_{t_l},t_l)\in\Delta(\calV)$ denotes the neural network's approximation to the clean-token posterior at coordinate $j$.
The head provides a direct clean-sequence completion distribution by fixing revealed tokens and independently sampling still-masked positions from the coordinate-wise clean-token posteriors $\mu_\theta^j(\cdot\mid\bar \by_{t_{l-1}},t_{l-1})$; we write $p^{\mathrm{head}}_\theta$ for this direct completion law, which is not generally a rollout conditional:
\begin{align}
  \label{eq:factorized-head}
  p^{\mathrm{head}}_\theta(\bx_0\mid\bar \by_{t_{l-1}})
  =
  \prod_{\mathclap{j:\,\bar y_{t_{l-1}}^j\neq\mathsf m}}\delta_{\bar y_{t_{l-1}}^j}(x_0^j)
  \prod_{\mathclap{j:\,\bar y_{t_{l-1}}^j=\mathsf m}} \mu_\theta^j(x_0^j\mid\bar \by_{t_{l-1}},t_{l-1}).
\end{align}
\begin{definition}[Prediction-head $h$-estimator]
  \label{def:head-estimator}
  \label{assumption:head-proxy}
  For every reachable candidate reverse state
  $\bar \by_{t_{l-1}}\in\widetilde{\calX}$, D-DOIT estimates the downstream
  value by the one-step clean-completion expectation under the factorized
  prediction-head distribution for each candidate state:
  \begin{align}
    \label{eq:h-rollout}
    \bar h_\theta(\bar \by_{t_{l-1}},t_{l-1})
    &=
    \E_{\bx_0\sim p^{\mathrm{head}}_\theta(\bx_0\mid\bar \by_{t_{l-1}})}
    \!\left[\phi(\bx_0)\right].
  \end{align}
\end{definition}
The estimator \eqref{eq:h-rollout} replaces the rollout value
$h(\bar \by_{t_{l-1}},t_{l-1})$ from \cref{def:discrete-doob-h-function} with a
single-step completion value. This is the main approximation that makes the
transform practical; \cref{app:proxy-discussion} discusses its bias relative to full future rollouts. For each candidate, we estimate \eqref{eq:h-rollout} with
one tempered completion
$\hat\bx_0\sim p^{\mathrm{head}}_{\theta,\gamma}(\cdot\mid\bar \by_{t_{l-1}})$:
\begin{align}
  \label{eq:h-hat-candidate}
  \hat h(\bar \by_{t_{l-1}},t_{l-1})
  &=\exp(R(\hat\bx_0)/\beta),
  \qquad
  \hat\bx_0\sim p^{\mathrm{head}}_{\theta,\gamma}(\cdot\mid\bar \by_{t_{l-1}}),
  \\
  \mu_{\theta,\gamma}^j(v\mid\bar \by_{t_{l-1}},t_{l-1})
  &\propto
  \mu_\theta^j(v\mid\bar \by_{t_{l-1}},t_{l-1})^{1/\gamma}.
\end{align}
The completion temperature $\gamma$ tempers the \emph{completion} of a candidate $\bar \by_{t_{l-1}}$ into a clean sequence. Smaller $\gamma$ sharpens the completion toward the model mode, while larger $\gamma$ smooths noisy predictions.
Because the reward is queried only on the sampled prediction-head completion $\hat\bx_0$, $R(\cdot)$ need not be differentiable, and no gradients are backpropagated through the reward oracle. Using $n_h$ independent completions per candidate averages
\begin{align}
  \label{eq:multi-completion}
  \hat h_{n_h}(\bar \by_{t_{l-1}},t_{l-1})
  &=\frac{1}{n_h}\sum_{i=1}^{n_h}\exp(R(\hat\bx_0^{(i)})/\beta),
  \qquad \hat\bx_0^{(i)}\sim p^{\mathrm{head}}_{\theta,\gamma}(\cdot\mid\bar\by_{t_{l-1}}),
\end{align}
which reduces Monte Carlo variance for this head expectation without removing its discrepancy from the rollout value. The reported main configurations use $n_h=1$.

\paragraph{Finite-candidate discrete $h$-transform.}
At reverse step $l$, D-DOIT first draws $M$ candidate next states
$\{\bar \by_{t_{l-1}}^{(m)}\}_{m=1}^M$ from a $\tau$-tempered masked discrete diffusion proposal:
\begin{align}
  \label{eq:tempered-proposal}
  \mu_{\theta,\tau}^j(v\mid\bar \bx_{t_l},t_l)
  \propto
  \mu_\theta^j(v\mid\bar \bx_{t_l},t_l)^{1/\tau},
  \qquad v\in\calV .
\end{align}
Both temperatures act on the same coordinate-wise prediction head $\mu_\theta^j$, but they play different roles: $\tau$ controls proposal diversity, while $\gamma$ controls the clean completion used for reward scoring. We draw each candidate $\bar \by_{t_{l-1}}^{(m)}$ from the masked discrete diffusion reverse kernel \eqref{eq:mdlm-transition} with $p_{\theta,\tau}$ in place of $p_\theta$, keeping the schedule factor $(\alpha_{l-1}-\alpha_l)/(1-\alpha_l)$ and the copy-through rule. All $M$ proposals share one forward pass, and each is completed once via \eqref{eq:h-hat-candidate}.
The next state $\bar \bx_{t_{l-1}}=\bar \by_{t_{l-1}}^{(m^*)}$ is drawn by
sampling the index $m^*$ over the $M$ candidates, each assigned probability
\begin{align}
  \label{eq:tilted-resample}
  w^{(m)}
  =\frac{\hat h(\bar \by_{t_{l-1}}^{(m)},t_{l-1})^\omega}
  {\sum_{r=1}^M \hat h(\bar \by_{t_{l-1}}^{(r)},t_{l-1})^\omega}.
\end{align}
With $\tau=1$, \eqref{eq:tilted-resample} is a finite-candidate
self-normalized approximation to the $h$-reweighted transition induced by
\eqref{eq:discrete-doob-kernel}, with $\omega$ controlling how sharply the
sampler follows the estimated downstream value. For one completion,
$\hat h^\omega=\exp(\omega R(\hat\bx_0)/\beta)$, so $\omega$ is not a separate
approximation. Every $(\omega,\beta)$ pair is equivalent to $\omega=1$ with
effective inverse temperature $\eta=\omega/\beta$, the parameter used in
\cref{app:tempered-reference}.
With $\tau>1$, the proposal is deliberately smoothed, so the weighted selection rule becomes a biased proposal approximation that we use to maintain candidate diversity. Retaining the original $P$-referenced target under a smoothed proposal would instead require the importance weight
\begin{align}
  \label{eq:corrected-weight}
  \widetilde w^{(m)}\propto
  \frac{P_{\theta,t_l}(\bar\by_{t_{l-1}}^{(m)}\mid\bar\bx_{t_l})}
       {Q^\tau_l(\bar\by_{t_{l-1}}^{(m)}\mid\bar\bx_{t_l})}
  h(\bar\by_{t_{l-1}}^{(m)},t_{l-1}),
\end{align}
where $Q^\tau_l$ is the proposal kernel, assuming its support covers the target transition. Our operating points omit this ratio, so it is a reference correction, not the executed algorithm. Even with it, a biased proxy, sharpening, and terminal best-of-$K$ selection would remain.
A guided step costs $1+M$ model-side NFE: one shared proposal pass and $M$ completion passes. It also costs $M$ reward evaluations; \cref{app:nfe-breakdown} gives the full accounting and cost conventions.

\paragraph{Best-of-$K$ refinement.}
Revealed tokens cannot be revised, so D-DOIT optionally duplicates each active trajectory into $K$ copies after split update $s_\spl$, counted in reverse updates. The copies continue independently, and the method keeps the highest-reward clean sequence within each group. For each original trajectory $b$, the selected output is
\begin{align}
  \label{eq:bok-select}
  \bx_0^{(b)}=\bx_0^{(b,k_b^*)},
  \qquad
  k_b^*=\argmax_{k\in\{1,\ldots,K\}}R(\bx_0^{(b,k)}).
\end{align}
The $K\times$ multiplier is paid only after the split, hedging early decisions without always multiplying the full trajectory cost. This terminal argmax changes the output law and cannot repair earlier reveals. Before the split a trajectory is shared; after it the $K$ copies update independently, and with $K=1$ or $s_\spl=L$ the method reduces to D-DOIT. \Cref{alg:doit-discrete,alg:doit-guided-step} in the appendix give the sampler and its per-step guided update.

\subsection{Two Configurations: D-DOIT-LBOK-S and D-DOIT-Q}
\label{sec:operating-points}

We report two operating points of the same sampler that differ only in budget allocation.
\textbf{D-DOIT-LBOK-S} uses a small per-step candidate budget plus late best-of-$K$ refinement for the practical lower-budget regime.
\textbf{D-DOIT-Q} allocates more inference budget to candidate evaluation and/or split-time refinement to test the quality frontier of the frozen backbone.
\Cref{tab:hyperparams} lists the benchmark-specific values, and \cref{sec:exp} evaluates both configurations against training-free and training-based baselines.

\section{Experiments}
\label{sec:exp}

% We evaluate D-DOIT on regulatory DNA design in \cref{sec:exp:drakes}, on cell-type-specific enhancer design in \cref{sec:exp:ctrldna}, and on protein inverse folding in \cref{sec:exp:protein}.
% \Cref{tab:method-comparison} contrasts D-DOIT with the closest reward-guided sampling baselines.
% Compared to CG, D-DOIT does not require reward backprop, so it accepts non-differentiable rewards.
% Compared to TDS/SMC, D-DOIT selects among per-step candidate next states rather than full trajectory histories, which decouples cost from chain length and avoids particle-history bookkeeping.
% D-DOIT-LBOK-S is the practical lower-budget operating point, while D-DOIT-Q relaxes the inference-time budget for higher quality without retraining the backbone.

In this section, we include the following results: (1) regulatory DNA enhancer design (\cref{sec:exp:drakes}); (2) cell-type-specific enhancer design (\cref{sec:exp:ctrldna}); (3) protein inverse folding (\cref{sec:exp:protein}); (4) contrast with the closest reward-guided sampling baselines (\cref{sec:exp:runtime}); (5) analysis of the estimation accuracy of $\hat\bx_0$ (\cref{sec:exp:diagnostics-main}); (6) analysis of the impact of different reward functions (\cref{sec:exp:robustness}).

% BEGIN INLINE tables/method_comparison
\begin{table}[!ht]
  \centering
  \caption{\textbf{Comparison of reward-guided sampling methods for discrete diffusion.}
  We contrast D-DOIT with our guidance baselines (\cref{sec:exp}); CG = classifier guidance, TDS = twisted diffusion sampler, SMC = sequential Monte Carlo, SVDD = soft value-based decoding.}
  \label{tab:method-comparison}
  \renewcommand{\arraystretch}{0.95}
  \resizebox{\textwidth}{!}{%
    \begin{tabular}{@{}lcccc@{}}
      \toprule
      \textbf{Property}
      & \textbf{CG}
      & \textbf{TDS/SMC}
      & \textbf{SVDD}
      & \textbf{D-DOIT (Ours)} \\
      \midrule
      Reward intervention
      & gradient (backward)
      & particle weights
      & soft value estimate
      & Doob $h$-transform on the discrete kernel \\
      Requires backprop
      & \ding{51}
      & \ding{55}
      & \ding{55}
      & \ding{55} \\
      Weighted selection unit
      & ---
      & $M$ particle trajectories
      & $M$ candidate states
      & $M$ candidate next states (same $\bx_l$) \\
      Weight formula
      & ---
      & path likelihood ratio
      & estimated soft value
      & transformed kernel via $\hat h^\omega$ \\
      Doob kernel tilt
      & \ding{55}
      & \ding{55}
      & \ding{55}
      & \ding{51} \\
      Candidate diversity
      & ---
      & different histories
      & same state start
      & same $\bx_l$ start \\
      \bottomrule
    \end{tabular}%
  }
\end{table}
% END INLINE tables/method_comparison

% Furthermore, we also provide xxx
% A compute-matched comparison must state whether its budget is backbone NFE, scored reward candidates, or wall-clock time.
% Held-out prediction scores assess design quality, not experimental activity. Diagnostics, oracle robustness, ablations, and full protocols are reported in the appendix, including \cref{sec:ablation,app:drakes-search-controls,app:ctrldna-search-controls,app:exp:protein}.

Furthermore, we conducted detailed ablation studies on the D-DOIT algorithm using the regulatory DNA sequence design task to analyze the relationship between its computational efficiency and performance, and compared it with other algorithms (\cref{app:exp:drakes,sec:ablation}). The task cell type specific enhancer design also features a detailed efficiency-performance analysis (\cref{app:exp:ctrldna}. Detailed experimental setups are also provided for the protein inverse folding task, serving as a complement to cross-modal transfer (\cref{app:exp:protein}).
\cref{tab:method-comparison} shows the differences between our algorithm and several other algorithms.

\subsection{Task 1: Regulatory DNA Sequence Design}
\label{sec:exp:drakes}
Regulatory DNA elements direct cell-type-specific gene expression.
They form an important target for gene therapy~\cite{gosai2023machine}.
We use the HepG2 enhancer design benchmark of \cite{wang2025drakes}: generate $200$\,bp enhancer sequences with high predicted HepG2 activity and natural enhancer statistics on a frozen masked discrete diffusion backbone.
The dataset, reward oracle, checkpoint, protocol, metrics, and baselines are documented in \cref{app:exp:drakes:setup,app:exp:drakes:metrics}.

% BEGIN INLINE tables/results
\begin{table}[!ht]
  \centering
  \footnotesize
  \renewcommand{\arraystretch}{0.94}
  \caption{\textbf{HepG2 enhancer design results.}
    We compare D-DOIT with training-free guidance and training-based finetuning baselines on DRAKES HepG2; all metrics are higher-is-better.
    Parentheses denote standard deviation across seeds, and \textbf{bold}/\underline{underline} mark best/second best per column.
  D-DOIT uses held-out eval/test seed protocols in \cref{app:drakes-lbok-heldout,app:drakes-heldout}.
  $^\dagger$SVDD-PM~\citep{li2025derivativefree} is the matched-time control from \cref{tab:drakes-search-controls} ($M=416$), not part of the original DRAKES sweep.}
  \label{tab:results_drakes}
  \resizebox{\textwidth}{!}{%
    \begin{tabular}{@{}lcccc@{}}
      \toprule
      Method & Pred-Activity (median)\,$\uparrow$ & ATAC-Acc\,$\uparrow$\,(\%) & 3-mer Corr\,$\uparrow$ & JASPAR Corr\,$\uparrow$ \\ \midrule
      \addlinespace[0.1em]
      \multicolumn{5}{@{}l}{\emph{Training-free methods}} \\
      \addlinespace[0.1em]
      Pre-trained & 0.17(0.04) & 1.5(0.2) & -0.061(0.034) & 0.249(0.015) \\
      CG & 3.30(0.00) & 0.0(0.0) & -0.065(0.001) & 0.212(0.035) \\
      SMC & 4.15(0.33) & 39.9(8.7) & 0.840(0.045) & 0.756(0.068) \\
      TDS & 4.58(0.49) & 40.0(14.1) & 0.817(0.097) & 0.798(0.044) \\
      SVDD-PM$^\dagger$ & 6.04(0.07) & 39.8(0.4) & 0.688(0.012) & 0.917(0.002) \\
      \textbf{D-DOIT-LBOK-S (Ours)} & 4.56(0.04) & 65.0(1.4) & 0.851(0.017) & \textbf{0.948(0.002)} \\
      \textbf{D-DOIT-Q (Ours)} & \underline{\MainResult(\MainResultUncertainty)} & 91.7(0.9) & \textbf{0.894(0.009)} & \underline{0.923(0.002)} \\ \midrule
      \addlinespace[0.1em]
      \multicolumn{5}{@{}l}{\emph{Training-based methods}} \\
      \addlinespace[0.1em]
      CFG & 5.04(0.06) & \underline{92.1(0.9)} & 0.746(0.001) & 0.864(0.011) \\
      DRAKES w/o KL & \textbf{6.44(0.04)} & 82.5(2.8) & 0.307(0.001) & 0.557(0.015) \\
      \textbf{DRAKES} & 5.61(0.07) & \textbf{92.5(0.6)} & \underline{0.887(0.002)} & 0.911(0.002) \\
      \bottomrule
    \end{tabular}%
  }
  % \vspace{0.1\baselineskip}
\end{table}
% END INLINE tables/results

\Cref{tab:results_drakes} reports quality metrics.
D-DOIT-Q outperforms every training-free baseline in \cref{tab:results_drakes} on every metric.
D-DOIT-Q reaches $\MainResult$ Pred-Activity, $91.7\%$ ATAC-Acc, $0.894$ 3-mer Corr, and $0.923$ JASPAR Corr.
It exceeds training-based DRAKES on Pred-Activity, 3-mer Corr, and JASPAR Corr, while remaining slightly lower on ATAC-Acc ($91.7\%$ vs.\ $92.5\%$).
The near-TDS-speed D-DOIT-LBOK-S matches TDS on Pred-Activity ($4.56$ vs.\ $4.58$) while lifting ATAC-Acc from $40.0\%$ to $65.0\%$, 3-mer Corr from $0.817$ to $0.851$, and JASPAR Corr from $0.798$ to $0.948$ over TDS.
At a similar measured wall-clock budget, Q improves all four metric means over SVDD-PM and improves activity, accessibility, and 3-mer correlation over Frozen Best-of-732; endpoint search retains higher JASPAR correlation.
\Cref{app:drakes-search-controls,tab:drakes-search-controls} gives the matched-time results and full protocol.
A guided step costs $1+M$ model forward passes plus $M$ reward evaluations.
The high-budget Q row is calibrated to endpoint sampling at approximately $10.6$ seconds per output on the same A100; \cref{app:nfe-breakdown} gives the accounting.

\subsection{Task 2: Cell-type-specific Enhancer Design}
\label{sec:exp:ctrldna}

Cell-type-specific enhancer design asks for sequences active in a target cell line and silent in off-target cell lines.
This task is a stricter test of controllable generation than \cref{sec:exp:drakes}.
We conduct experiments on the Ctrl-DNA enhancer benchmark of \cite{chen2025ctrldna}.
The benchmark spans three cell lines (HepG2, K562, SK-N-SH) and supplies cell-type-specific Enformer-based reward oracles.
The dataset, composite specificity reward, per-cell-line oracles, diffusion training details, and the baseline-tuning protocol are documented in \cref{app:exp:ctrldna:setup,app:exp:ctrldna:baseline-tuning}.

We report specificity ($\Delta R$), Motif Corr, and set-level diversity; $\Delta R$ is the target-cell oracle score minus the average off-target score.
Baselines include the training-free methods as in \cref{sec:exp:drakes}, regLM~\cite{lal2024reglm}, RL baselines~\cite{chen2025ctrldna}, and directed evolution; full metric definitions and baseline tuning are documented in the appendix, \cref{app:exp:ctrldna:metrics,app:exp:ctrldna:baseline-tuning}.

\Cref{tab:results_ctrldna} reports three-cell-line averages. Among the diffusion-guidance baselines, D-DOIT clearly improves specificity: D-DOIT-LBOK-S raises $\Delta R$ from $0.002$ to $0.033$ over TDS while improving Motif Corr from $0.573$ to $0.648$ at near-TDS wall-clock.
D-DOIT-Q pushes $\Delta R$ to $0.119$ while trading motif preservation for reward, giving the high-specificity point; LBOK-S remains the lower-budget point.
Under the matched-time control, endpoint search also improves specificity and captures much of this gain, unlike the larger activity and accessibility separation on DRAKES. Q's mean differences from Frozen Best-of-708 are small and are not claimed statistically significant, so the specificity gain is not unique to D-DOIT here.
Wall-clock, per-cell, and diversity details and the matched-time control are in \cref{app:exp:ctrldna:analysis,app:ctrldna-search-controls}.

\subsection{Task 3: Protein Inverse Folding}
\label{sec:exp:protein}
To test a different sequence space and reward, we apply D-DOIT to the DRAKES
protein inverse-folding benchmark~\citep{wang2025drakes}, using its reward
oracle and evaluation metrics. Here the target backbone is fixed and the
sampler generates amino-acid sequences. We evaluate predicted stability and
agreement between the predicted structure and that backbone, with the full
cross-method comparison including every reported metric for this benchmark in
\cref{tab:protein-results}.

Across all $12$ test backbones, three seeds, and $128$ sequences per backbone
and seed, D-DOIT achieves joint success of $68.1\pm3.5\%$, compared with
the published SMC and TDS results of $63.6\%$ and $62.9\%$
(\cref{tab:protein-results}). The configuration $M=52$, $K=8$, $n_h=1$ was fixed
before test evaluation, with no post-test tuning, and each entry reports the
mean and sample SD across the three seed-level metrics.
D-DOIT has the highest predicted stability and
joint success among the six methods, while TDS has better structural fidelity.
The other five rows are taken from DRAKES Table~2, not rerun or compute-matched
controls. \Cref{app:exp:protein} gives the full protocol. This supports transfer
to a second biological sequence-design task, not experimental protein validation.

% BEGIN INLINE tables/task_results_small
\begin{table}[H]
  \centering
  \caption{\textbf{Cell-type-specific enhancer design (left) and protein inverse folding (right).}
    Left: means over three cell lines, SD in parentheses; $\Delta R$ is target minus average off-target reward; training-based references are in \cref{tab:results_ctrldna_per_cell,app:exp:ctrldna:analysis}.
    Right: DRAKES protein benchmark; Pred-ddG and joint success are higher-is-better, scRMSD lower-is-better, SDs in \cref{tab:protein-results-full}.
    Bold marks the best mean per column, not significance.}
  \label{tab:results_ctrldna}
  \label{tab:protein-results}
  \begin{minipage}[t]{0.5\textwidth}
    \centering
    \scriptsize
    \renewcommand{\arraystretch}{1.00}
    \begin{tabular}{@{}lcc@{}}
      \toprule
      Method & $\Delta R$\,$\uparrow$ & Motif Corr.\,$\uparrow$ \\
      \midrule
      \multicolumn{3}{@{}l}{\emph{Training-free methods}} \\
      Pre-trained & $0.000(0.003)$ & $\underline{0.937(0.009)}$ \\
      CG & $0.000(0.003)$ & $\mathbf{0.938(0.009)}$ \\
      SMC & $0.000(0.010)$ & $0.587(0.114)$ \\
      TDS & $0.002(0.009)$ & $0.573(0.078)$ \\
      Frozen Best-of-708 & $0.108(0.028)$ & $0.490(0.165)$ \\
      D-DOIT-LBOK-S & $0.033(0.025)$ & $0.648(0.046)$ \\
      D-DOIT-Q & $\underline{0.119(0.063)}$ & $0.498(0.070)$ \\
      \bottomrule
    \end{tabular}
  \end{minipage}\hfill
  \begin{minipage}[t]{0.5\textwidth}
    \centering
    \scriptsize
    \renewcommand{\arraystretch}{1.00}
    \setlength{\tabcolsep}{3pt}
    \begin{tabular}{@{}lccc@{}}
      \toprule
      Method & Pred-ddG\,$\uparrow$ & scRMSD\,$\downarrow$ & Success\,(\%)$\,\uparrow$ \\
      \midrule
      \multicolumn{4}{@{}l}{\emph{Training-free methods}} \\
      Pre-trained & $-0.544$ & $0.849$ & $34.4$ \\
      CG & $-0.561$ & $0.839$ & $34.7$ \\
      SMC & $0.659$ & $0.841$ & $63.6$ \\
      TDS & $0.674$ & $\mathbf{0.834}$ & $62.9$ \\
      \textbf{D-DOIT (Ours)} & $\mathbf{0.870}$ & $0.872$ & $\mathbf{68.1}$ \\
      \midrule
      \multicolumn{4}{@{}l}{\emph{Training-based methods}} \\
      CFG & $-1.186$ & $3.146$ & $1.3$ \\
      \bottomrule
    \end{tabular}
  \end{minipage}
\end{table}
% END INLINE tables/task_results_small

% This existing analysis is part of the main empirical story; the appendix
% retains the detailed NFE accounting and diagnostic sweeps.
\subsection{Compute--quality Trade-off}
\label{sec:exp:runtime}

To compare compute across methods, we account for oracle reward calls in
addition to model evaluations. A guided update evaluates a shared proposal
head, $M$ candidate heads, and one reward per completed candidate; branching
adds work only after the split.

% BEGIN INLINE tables/speed
\begin{wraptable}{r}{0.4\textwidth}
\vspace{-1.5em}
  \centering
  \caption{\textbf{Per-output DRAKES work.}
  Full-window guidance, $L=128$, one completion per candidate.}
  \label{tab:speed}
  \footnotesize
  \setlength{\tabcolsep}{5pt}
  \renewcommand{\arraystretch}{1.35}
  \begin{tabular}{@{}lcc@{}}
    \toprule
    Quantity & LBOK-S & Q \\
    \midrule
    $(M,K,s_\spl)$ & $(9,4,105)$ & $(52,8,0)$ \\
    Model NFE & $1{,}970$ & $54{,}272$ \\
    Candidate completions & $1{,}773$ & $53{,}248$ \\
    \bottomrule
  \end{tabular}
\end{wraptable}
% END INLINE tables/speed

The lower-budget DRAKES setting uses $1{,}970$ model evaluations and
$1{,}773$ candidate completions per output, with $1{,}777$ reward evaluations
after ranking the four final branches. Terminal ranking adds up to $K$
queries when scores are not cached. In the $640$-output diagnostic,
$1{,}290$ vectorized oracle calls score $1{,}137{,}280$ sequences, so physical
calls and scored sequences are distinct cost units. Full accounting is given
in \Cref{app:nfe-breakdown}.

Model NFE alone therefore does not determine total compute. In
\cref{tab:reward-call-accounting}, Frozen Best-of-15 has a similar model
budget ($1{,}920$ versus $1{,}970$ evaluations per output) but only $15$
rather than $1{,}777$ oracle evaluations. Its $640$-output job is faster
($97.8$ versus $327.8$ seconds) but has lower Pred-Activity
($2.219$ versus $4.491$). Thus, this comparison approximately matches model
budget, not wall-clock time. The calibration in
\cref{tab:drakes-search-controls} instead matches endpoint sampling to Q at
approximately $10.6$ seconds per output.

LBOK-S remains the near-TDS-runtime operating point ($1.16\times$ in the
main timing study), whereas Q allocates substantially more compute to quality.
For Ctrl-DNA, Q and Frozen Best-of-708 take $5.43$ and $5.51$ seconds per
output on the same A100. Because absolute timings depend on batching and
memory use, we report each study under its own timing protocol. D-DOIT
generation costs recur for new samples, and comparisons with training-based
methods concern sample quality rather than matched or amortized training cost.
\subsection{Completion Quality and Search Diagnostics}
\label{sec:exp:diagnostics-main}

A crucial aspect of estimating $h$ is how to obtain $\hat\bx_0$; while a common approach to maintaining speed is to estimate $\hat\bx_0$ directly from $x_t$ in a single step, a more refined method involves obtaining $\hat\bx_0$ through iterative steps starting from $x_t$. To better compare these two approaches, we investigated the performance of different estimation methods.
Increasing completions per candidate from $n_h=1$ to $2$ and $3$ raises final
guidance reward $R$ from $5.272\pm0.057$ to $5.764\pm0.033$ and
$5.978\pm0.026$, at $41\%$ and $81\%$ additional wall time; we retain
$n_h=1$ as the default. Direct and sequential completions have similar
$h$-value distributions (Wasserstein-1 $0.060$--$0.124$;
\cref{sec:ablation:proxy}), supporting practical use without exact-target
equivalence.

\subsection{Oracle Robustness}
\label{sec:exp:robustness}
The choice of reward function is independent of the D-DOIT algorithm itself, yet it significantly influences the algorithm's performance. A more accurate reward function provides better guidance for the D-DOIT algorithm. To examine this, we designed various reward functions and observed their guiding effects.
Replacing $R$ with Enformer, CNN, or MLP predictions improves all four metrics
over unguided sampling (\cref{tab:oracle-guidance}); Pred-Activity
reaches $4.491$, $2.203$, and $1.544$, versus $0.371$ for the frozen model.
Held-out rescoring preserves these gains (\cref{tab:oracle-evaluators}), while
evaluator scales remain incomparable. The controls support robustness across
tested predictors without establishing experimental biological activity and are
distinct from the high-budget $\MainResult\pm\MainResultUncertainty$ result.

\section{Conclusion and Discussion}
\label{sec:conclusion}
We introduced D-DOIT, a training-free sampler that applies Doob's
$h$-transform directly to frozen masked discrete diffusion models.
Using a prediction-head $h$-estimator and finite-candidate Doob-weighted
selection, D-DOIT supports black-box rewards without reward gradients or
backbone retraining. Across DRAKES HepG2, Ctrl-DNA, and protein inverse
folding, D-DOIT outperforms training-free guidance baselines, although its
highest-quality results require a larger compute budget.

The method has three main limitations. First, practical approximations,
including temperature scaling, sharpening, and best-of-$K$ selection, may
introduce bias and do not retain the exact-target guarantee; late branching
also cannot revise revealed tokens. Second, evaluation relies on predictive
biological benchmarks and requires experimental validation. Third, the
highest-quality setting incurs substantial recurring inference cost, leaving
the reward--compute trade-off as a practical challenge. Overall, Doob-style
kernel reweighting provides a general framework for guiding discrete
generative models with sequence-level rewards.

\clearpage
\bibliographystyle{unsrtnat}
{\hbadness=10000
\bibliography{refs}}

\clearpage
\normalsize
\appendix
\label{sec:append}
\renewcommand{\arraystretch}{1.18}
\setlength{\tabcolsep}{6pt}

\part*{Appendix}
\addcontentsline{toc}{part}{Appendix}
{\startcontents[sections]

 \setlength{\parskip}{3.5pt}
 \printcontents[sections]{ }{1}{}
}

\clearpage
% Start every top-level appendix section on a new page.
\let\appendixoriginalsection\section
\renewcommand{\section}{\clearpage\appendixoriginalsection}
\clearpage
\section{Related Work}
\label{app:related}

We review related work on reward-guided adaptation of discrete diffusion and flow models, as well as regulatory DNA sequence design.

\paragraph{Reward-guided adaptation of discrete diffusion and flow models.}
Reward-guided generation has proved effective for aligning language models \cite{guo2025deepseek, brown2020language, wang2026sentinel, wang2026opera, wang2025customer, wang2026trajectory2task} and diffusion models \cite{ho2020denoising, hu2024statistical, hu2025statistical} with task-specific objectives.
In discrete generative models, however, guidance must act on categorical transitions over a finite vocabulary.
Early discrete denoising diffusion models use structured categorical or absorbing-state corruption~\citep{austin2021structured,hoogeboom2021argmax,campbell2022continuous}; masked discrete diffusion models scale a simple absorbing-mask objective~\citep{sahoo2024simple,shi2024simplified}; score-entropy and ratio-matching parameterizations form a non-masked alternative~\citep{lou2024sedd,ou2024radd}; and discrete flow matching extends the same idea to a flow framework~\citep{campbell2024discrete,gat2024discrete}. Reward adaptation for these models can be training-based or training-free. Training-based methods include DRAKES, which finetunes the generator through reward backpropagation~\citep{wang2025drakes}, trajectory-level alignment~\citep{han2025trajectory}, and reinforcement learning (RL)-style finetuning methods developed for continuous diffusion~\citep{black2023ddpo,fan2023dpok,clark2024draft,domingoenrich2024adjoint}; these methods change backbone weights and need a differentiable reward. Training-free methods instead steer a frozen backbone at inference time, including Sequential Monte Carlo and twisted sampling~\citep{wu2024practical,trippe2023diffusion,cardoso2024mcgdiff}, posterior sampling for inverse problems~\citep{chung2023dps}, soft value-based decoding~\citep{li2025derivativefree}, reward-guided iterative refinement~\citep{uehara2025reward}, clean-sample Markov chains~\citep{phunyaphibarn2026reward}, simple guidance for masked discrete diffusion~\citep{schiff2024simple}, denoising posterior prediction~\citep{rectorbrooks2024ddpp}, plug-and-play control~\citep{frey2024plugplay}, and inference-time scaling~\citep{singhal2025general,ma2025inference}. These methods steer a frozen backbone, but they do not directly tilt the discrete reverse kernel toward an optimal proposal.

Comparing with three representative efforts highlights the distinct role of D-DOIT.
DRAKES~\citep{wang2025drakes} provides a strong reward-optimization baseline, but it finetunes the backbone and requires a differentiable reward.
Soft value-based decoding~\citep{li2025derivativefree} selects candidates using importance scores on intermediate latents, but it does not tilt the reverse transition kernel itself.
Simple masked-diffusion guidance~\citep{schiff2024simple} introduces a heuristic guidance term, but lacks an explicit optimality target.
In contrast, D-DOIT is the first inference-time method to implement an executable estimator of Doob's $h$-transform for the discrete reverse kernel~\citep{doob1984classical,zhu2026trainingfreeadaptationdiffusionmodels,heng2021simulating,denker2024deft,debortoli2021schrodinger}.
It scores candidates through the resulting $h$-tilted kernel, combines a frozen masked discrete diffusion backbone with a black-box reward, and further supports optional best-of-$K$ refinement over the same frozen prior.
Because finite-candidate value-weighted selection also appears in soft
value-based decoding and related methods, we do not claim that local candidate
selection is unique to D-DOIT. CDM~\citep{kim2026cdm} learns a twist to amortize
value estimation, and DPRM~\citep{bu2026dprm} applies Doob-derived corrections
to token ordering while retaining a host token sampler; CSMC~\citep{phunyaphibarn2026reward}
instead searches over clean samples. Our distinction is the explicit
categorical reverse-kernel $h$-transform as the reference object, evaluated
with a frozen prediction-head proxy rather than a learned value network.

\paragraph{Regulatory DNA sequence design.}
Regulatory DNA design aims to generate sequences that score well under predictive assays and stay faithful to natural regulatory grammar. Oracle-guided optimization with experimental validation drives early machine-guided enhancer design~\citep{taskiran2024cell,gosai2023machine, su2026discrete}; autoregressive regulatory language models~\citep{lal2024reglm} and genomic foundation models~\citep{nguyen2023hyenadna,schiff2024caduceus,dalla2024nucleotide, zhou2025genomeocean, wu2025genome, zhou2025dnabert} provide strong sequence priors; and diffusion-based generators specialize these priors through Dirichlet diffusion and score-entropy variants~\citep{stark2024dirichlet,sarkar2024designing,avdeyev2023dirichlet}. For reward-optimization benchmarks, DRAKES evaluates HepG2 enhancer generation through activity, accessibility, $k$-mer, and motif-structure metrics~\citep{wang2025drakes}; Ctrl-DNA studies cell-type-specific enhancer design through constrained reinforcement learning~\citep{chen2025ctrldna}; and concurrent work combines discrete diffusion with tree search for cell-type-specific regulatory DNA design under related biological reward objectives~\citep{awasthi2026conditional}. The strongest methods on these benchmarks share one limit: they reach high reward only through backbone updates, policy training, or expanded inference search.

Comparing with three representative efforts highlights the efficiency and generality of D-DOIT.
DRAKES~\citep{wang2025drakes} finetunes the backbone and requires a differentiable reward.
Ctrl-DNA~\citep{chen2025ctrldna} achieves cell-type specificity through constrained RL with reward signals shaped for policy optimization.
The concurrent tree-search method of \citet{awasthi2026conditional} remains training-free at inference time, but expands a search tree over partial sequences and increases the per-step sampling cost.
In contrast, D-DOIT is the first method to compete on both benchmarks without policy training.
It trades modest inference-time compute for reward gains while preserving near-base per-step cost.

% \paragraph{Protein sequence design.}
% Hidden-state gradient guidance steers discrete diffusion for protein sequence optimization~\citep{gruver2023protein}, and reward-guided iterative refinement adapts diffusion samples at test time for protein and DNA design~\citep{uehara2025reward}. In inverse folding, the target backbone is fixed and the generated object is an amino-acid sequence. DRAKES studies this setting through reward-based finetuning and evaluates both predicted stability and structural fidelity~\citep{wang2025drakes}. Our protein experiment uses its benchmark to test frozen-model reward guidance beyond DNA. The published baseline rows provide a literature comparison, not matched-compute reruns; predicted stability and joint stability--structure success remain distinct from experimental protein validation.

\FloatBarrier
\clearpage
\section{Proof of Theorem~\ref{thm:discrete-doob-h-transform}}
\label{app:proof-discrete-doob-htransform}
\begin{proof}
  Let
  \begin{align*}
    \phi(\bx_0)
    =
    \exp(R(\bx_0)/\beta),
    \qquad
    Z_\beta
    =
    \E_{p_\theta}\left[\phi(\bX_0)\right].
  \end{align*}
  All identities below are understood on the support of the pretrained reverse
  chain.
  Let $\mathbb P_\theta$ denote the path law of the pretrained reverse chain,
  \begin{align*}
    \mathbb P_\theta
    (
      \bar \bx_{t_0},\ldots,\bar \bx_{t_L}
    )
    =
    p_{\theta,t_L}(\bar \bx_{t_L})
    \prod_{l=1}^L
    P_{\theta,t_l}
    (
      \bar \bx_{t_{l-1}}
      \mid
      \bar \bx_{t_l}
    ).
  \end{align*}
  Define the reward-tilted path law by
  \begin{align*}
    \frac{\dd \mathbb P_\theta^h}{\dd \mathbb P_\theta}
    (
      \bar \bX_{t_0},\ldots,\bar \bX_{t_L}
    )
    =
    \frac{\phi(\bX_0)}{Z_\beta}.
  \end{align*}
  We now compute the one-step reverse transition under this tilted path law. For
  $l=1,\ldots,L$,
  \begin{align*}
    &\mathbb P_\theta^h
    \left(
      \bar \bX_{t_{l-1}}=\bar \bx_{t_{l-1}}
      \mid
      \StateSymbol=\bar \bx_{t_l}
    \right)
    \\
    &=
    \frac{
      \E_{p_\theta}
      \left[
        \mathbf 1\{
          \bar \bX_{t_{l-1}}=\bar \bx_{t_{l-1}}
        \}
        \phi(\bX_0)
        \mid
        \StateSymbol=\bar \bx_{t_l}
      \right]
    }{
      \E_{p_\theta}
      \left[
        \phi(\bX_0)
        \mid
        \StateSymbol=\bar \bx_{t_l}
      \right]
    }
    \\
    &=
    P_{\theta,t_l}
    \left(
      \bar \bx_{t_{l-1}}
      \mid
      \bar \bx_{t_l}
    \right)
    \frac{
      \E_{p_\theta}
      \left[
        \phi(\bX_0)
        \mid
        \bar \bX_{t_{l-1}}=\bar \bx_{t_{l-1}},
        \StateSymbol=\bar \bx_{t_l}
      \right]
    }{
      h(\bar \bx_{t_l},t_l)
    }.
  \end{align*}
  By the Markov property of the reverse chain, once
  $\bar \bX_{t_{l-1}}$ is given, the future evolution toward clean time is
  independent of $\StateSymbol$. Hence
  \begin{align*}
    \E_{p_\theta}
    \left[
      \phi(\bX_0)
      \mid
      \bar \bX_{t_{l-1}}=\bar \bx_{t_{l-1}},
      \StateSymbol=\bar \bx_{t_l}
    \right]
    =
    h(\bar \bx_{t_{l-1}},t_{l-1}).
  \end{align*}
  Therefore
  \begin{align*}
    &\mathbb P_\theta^h
    \left(
      \bar \bX_{t_{l-1}}=\bar \bx_{t_{l-1}}
      \mid
      \StateSymbol=\bar \bx_{t_l}
    \right)=
    P_{\theta,t_l}
    \left(
      \bar \bx_{t_{l-1}}
      \mid
      \bar \bx_{t_l}
    \right)
    \frac{
      h(\bar \bx_{t_{l-1}},t_{l-1})
    }{
      h(\bar \bx_{t_l},t_l)
    }.
  \end{align*}
  This proves that the reward-tilted path law induces the Doob-transformed
  reverse kernel in \eqref{eq:discrete-doob-kernel}. In particular, the kernel is
  normalized because
  \begin{align*}
    &\sum_{\bar \bx_{t_{l-1}}}
    P_{\theta,t_l}
    \left(
      \bar \bx_{t_{l-1}}
      \mid
      \bar \bx_{t_l}
    \right)
    \frac{
      h(\bar \bx_{t_{l-1}},t_{l-1})
    }{
      h(\bar \bx_{t_l},t_l)
    }=
    \frac{
      \E_{p_\theta}
      \left[
        h(\bar \bX_{t_{l-1}},t_{l-1})
        \mid
        \StateSymbol=\bar \bx_{t_l}
      \right]
    }{
      h(\bar \bx_{t_l},t_l)
    }
    =
    1,
  \end{align*}
  where the last equality follows from the tower property and the definition of
  $h$.

  It remains to compute the marginals. For any $l=0,\ldots,L$,
  \begin{align*}
    p_{\theta,t_l}^h(\bar \bx_{t_l})
    &=
    \mathbb P_\theta^h
    (
      \StateSymbol=\bar \bx_{t_l}
    )
    \\
    &=
    \frac{
      \E_{p_\theta}
      \left[
        \mathbf 1\{
          \StateSymbol=\bar \bx_{t_l}
        \}
        \phi(\bX_0)
      \right]
    }{
      Z_\beta
    }
    \\
    &=
    \frac{
      p_{\theta,t_l}(\bar \bx_{t_l})
      \E_{p_\theta}
      \left[
        \phi(\bX_0)
        \mid
        \StateSymbol=\bar \bx_{t_l}
      \right]
    }{
      Z_\beta
    }
    \\
    &=
    \frac{
      p_{\theta,t_l}(\bar \bx_{t_l})
      h(\bar \bx_{t_l},t_l)
    }{
      Z_\beta
    }.
  \end{align*}

  In the fully masked terminal setting, $\bar \bX_{t_L}$ is deterministic. Hence
  \begin{align*}
    h(\bar \bX_{t_L},t_L)
    =
    \E_{p_\theta}\left[\phi(\bX_0)\right]
    =
    Z_\beta,
  \end{align*}
  so the transformed terminal distribution coincides with the pretrained
  terminal distribution.

  Finally, since $\bar \bX_{t_0}=\bX_0$,
  \begin{align*}
    h(\bx_0,0)
    =
    \E_{p_\theta}
    \left[
      \phi(\bX_0)
      \mid
      \bX_0=\bx_0
    \right]
    =
    \phi(\bx_0)
    =
    \exp(R(\bx_0)/\beta).
  \end{align*}
  Therefore
  \begin{align*}
    p_{\theta,0}^h(\bx_0)
    =
    \frac{
      p_{\theta,0}(\bx_0)
      \exp(R(\bx_0)/\beta)
    }{
      Z_\beta
    }
    =
    \pi_\beta(\bx_0).
  \end{align*}
  This completes the proof.
\end{proof}

\paragraph{A separate tempered reference chain.}
\label{app:tempered-reference}

The theorem above retains the original frozen $P_\theta$ reference and target
$\pi_\beta$. The following proposition describes a different reference law; it
neither replaces that target nor identifies the executed estimator with an exact
Doob value.

\begin{proposition}[Doob $h$-transform relative to a tempered reference]
  Fix $\tau>0$. Let $Q^\tau_l$ be the reverse kernel obtained from
  $P_{\theta,t_l}$ by replacing $\mu_\theta$ with the normalized
  $\mu_{\theta,\tau}$ in \eqref{eq:tempered-proposal}, while preserving the
  reveal schedule and copy rule. Let $\mathbb Q^\tau$ be the path law using
  these kernels at every reverse step and starting from the deterministic
  all-mask state $\bar\bx_{t_L}=\mathsf m^N$.
  Assume the chain ends in $\calX$, and let $q^\tau_0$ be its clean-time law.
  For $\eta>0$, suppose
  $0<Z_{\tau,\eta}:=\E_{\mathbb Q^\tau}[\exp(\eta R(\bX_0))]<\infty$.
  Define the true reference-chain value, on reachable states, by
  \begin{align*}
    h_{\tau,\eta}(\bar\bx,t_l)
    &=\E_{\mathbb Q^\tau}
    [\exp(\eta R(\bX_0))\mid\StateSymbol=\bar\bx].
  \end{align*}
  Then the kernels
  \begin{align*}
    Q^{\tau,h}_l(\bar\by\mid\bar\bx)
    &=Q^\tau_l(\bar\by\mid\bar\bx)
    \frac{h_{\tau,\eta}(\bar\by,t_{l-1})}
    {h_{\tau,\eta}(\bar\bx,t_l)}
  \end{align*}
  are normalized. Starting from the same deterministic mask, their clean-time
  marginal is
  \begin{align*}
    q^{\tau,h}_0(\bx_0)
    &=\frac{q^\tau_0(\bx_0)\exp(\eta R(\bx_0))}{Z_{\tau,\eta}}.
  \end{align*}
\end{proposition}

\begin{proof}
  The Markov and tower properties give
  \begin{align*}
    h_{\tau,\eta}(\bar\bx,t_l)
    &=\sum_{\bar\by}Q^\tau_l(\bar\by\mid\bar\bx)
      h_{\tau,\eta}(\bar\by,t_{l-1}),
  \end{align*}
  which proves normalization and positivity on the reference support.
  At clean time $h_{\tau,\eta}(\bx_0,0)=\exp(\eta R(\bx_0))$;
  at the deterministic initial mask
  $h_{\tau,\eta}(\mathsf m^N,t_L)=Z_{\tau,\eta}$.
  Consequently, the likelihood ratio of the transformed path law to
  $\mathbb Q^\tau$ telescopes to
  \begin{align*}
    \prod_{l=1}^{L}
    \frac{h_{\tau,\eta}(\bar\bx_{t_{l-1}},t_{l-1})}
    {h_{\tau,\eta}(\bar\bx_{t_l},t_l)}
    &=\frac{\exp(\eta R(\bx_0))}{Z_{\tau,\eta}}.
  \end{align*}
  Marginalizing this identity over the intermediate states yields the claimed
  clean-time law.
\end{proof}

For $\tau=1$, $\eta=1/\beta$ this recovers the $P$-reference theorem.
Otherwise the target changes; the practical proxy, windowing, and terminal
selection do not inherit its exactness.

\FloatBarrier
\clearpage
\section{Continuous-time Doob's \texorpdfstring{$h$}{h}-transform on Score-based Diffusion}
\label{app:continuous-doob}

We record the standard continuous-time derivation for comparison with the discrete construction.
Consider a forward score-based diffusion on $\R^d$,
\begin{align*}
  \dd \bX_t = f(\bX_t,t)\,\dd t + g(t)\,\dd \bW_t,
  \qquad t \in [0,T],
\end{align*}
with marginal density $p_t$ and $\bX_0 \sim p_0$.
The reverse-time process admits the well-known form
\begin{align*}
  \dd \bX_t = \bigl(f(\bX_t,t)-g(t)^2\nabla\log p_t(\bX_t)\bigr)\,\dd t + g(t)\,\dd \overline{\bW}_t,
\end{align*}
where $\overline{\bW}_t$ is reverse-time Brownian motion~\citep{anderson1982reverse}.
For the continuous analogue of the reward-tilted target in \eqref{eq:tilted-target}, write the base terminal density as $p_0$ and denote the target by $\pi(\bx_0) \propto p_0(\bx_0)\phi(\bx_0)$.
Define the $h$-function
\begin{align*}
  h_t(\bx) = \E\big[\phi(\bX_0)\,\big|\,\bX_t = \bx\big].
\end{align*}
Doob's $h$-transform replaces $p_t$ with the tilted marginal $\tilde p_t \propto p_t h_t$, which gives the tilted reverse drift
\begin{align*}
  \bigl(f(\bx,t)-g(t)^2\nabla\log p_t(\bx)\bigr) - g(t)^2\nabla\log h_t(\bx)
  = f(\bx,t)-g(t)^2\nabla\log\bigl(p_t(\bx) h_t(\bx)\bigr).
\end{align*}
Equivalently, the score is shifted by an additive correction:
\begin{align*}
  \nabla\log p_t(\bx) \;\longrightarrow\; \nabla\log p_t(\bx) + \nabla\log h_t(\bx).
\end{align*}
With initialization from the tilted marginal at time $T$, the terminal law of the tilted reverse process is $\pi$~\citep{zhu2026trainingfreeadaptationdiffusionmodels,heng2021simulating,denker2024deft}; using the base marginal at $T$ is exact only when that marginal is unchanged by the tilt.
The continuous construction requires sufficient regularity of $h_t$ for $\nabla\log h_t$ to exist, not differentiability of $R$ itself.
Under suitable smoothing and integrability conditions, differentiating a transition density can give this correction even for a black-box reward.
The structural distinction for masked discrete diffusion is instead that $\widetilde{\calX}$ is categorical and has no Euclidean score field.
\Cref{sec:method} therefore formulates a discrete kernel-level analogue, where the tilt enters as a multiplicative reweighting of categorical transition probabilities.

\FloatBarrier
\section{On the Prediction-head \texorpdfstring{$h$}{h}-estimator}
\label{app:proxy-discussion}

\Cref{def:head-estimator} is the executable counterpart of the rollout h-function used by the exact discrete transform.
For attention-based masked discrete diffusion models, it is generally an estimator rather than an exact identity.
For a candidate state $\bar \by_{t_{l-1}}$ at time $t_{l-1}$, the exact quantity is
\begin{align*}
  h(\bar \by_{t_{l-1}},t_{l-1}) = \E_{p_\theta}[\phi(\bX_0) \mid \bar \bX_{t_{l-1}} = \bar \by_{t_{l-1}}],
\end{align*}
where the expectation runs over the pre-trained reverse chain from time $t_{l-1}$ down to clean time, conditional on $\bar \bX_{t_{l-1}} = \bar \by_{t_{l-1}}$.
At each subsequent reverse step $l' \in \{l-1, l-2, \ldots, 1\}$, the reverse kernel \eqref{eq:mdlm-transition} reveals each still-masked position from $\mu_\theta^j(x_0^j \mid \bar \bx_{t_{l'}},t_{l'})$ and treats positions as conditionally independent.
The copy flag freezes any position that the chain has already revealed.
The exact rollout marginal of $\bX_0$ given $\bar \bX_{t_{l-1}} = \bar \by_{t_{l-1}}$ therefore mixes per-position predictions at the \emph{intermediate} states $\bar \bx_{t_{l'}}$ that the chain visits, which carry more revealed context than $\bar \by_{t_{l-1}}$.

The factorized head \eqref{eq:factorized-head} evaluates the clean completion from the current state $\bar \by_{t_{l-1}}$ itself, giving
\begin{align}
  p^{\mathrm{head}}_\theta(\bx_0 \mid \bar \by_{t_{l-1}})
  = \prod_{j:\, \bar y_{t_{l-1}}^j = \mathsf m} \mu_\theta^j(x_0^j \mid \bar \by_{t_{l-1}},t_{l-1}) \cdot \prod_{j:\, \bar y_{t_{l-1}}^j \neq \mathsf m} \delta_{\bar y_{t_{l-1}}^j}(x_0^j).
\end{align}
A sufficient condition for this distribution to coincide with the full rollout marginal is that each still-masked coordinate's head prediction remains invariant across all reachable subsequent contexts and times until that coordinate is revealed.
For attention-based masked discrete diffusion models this invariance need not hold. The direct completion law therefore generally differs from the multi-step rollout law, even with arbitrarily many completion samples.
To separate sampling noise from this distributional mismatch, let $n_h$ denote the number of independent completions of a fixed candidate (reserving $N$ for sequence length), as in \eqref{eq:multi-completion}, and write
\begin{align*}
  \hat h_{n_h}(\bar\by,t)
  &=\frac{1}{n_h}\sum_{i=1}^{n_h}\exp(R(\hat\bx_0^{(i)})/\beta),
  \qquad \hat\bx_0^{(i)}\sim p^{\mathrm{head}}_{\theta,\gamma}(\cdot\mid\bar\by),\\
  \bar h_{\theta,\gamma}(\bar\by,t)
  &=\E_{p^{\mathrm{head}}_{\theta,\gamma}(\cdot\mid\bar\by)}[\exp(R(\bX_0)/\beta)].
\end{align*}
The executed single-completion estimator has $n_h=1$.
As $n_h\to\infty$, the sample average converges to $\bar h_{\theta,\gamma}$, not generally to the rollout value $h$; at $\gamma=1$ it recovers the head expectation in \eqref{eq:h-rollout}.

\begin{remark}[Practical tempering and selection bias]
  Completion temperature $\gamma\neq1$ changes the direct completion law and hence its reward expectation; it does not merely reduce Monte Carlo noise.
  Proposal temperature $\tau\neq1$ changes the candidate kernel, denoted $Q^\tau_l$, relative to $P_{\theta,t_l}$.
  With independent candidate draws and completion randomness, the large-$M$ limit of the uncorrected selection rule at a fixed current state $\bar\bx$ is proportional to
  \begin{align*}
    Q^\tau_l(\bar\by\mid\bar\bx)\,
    \E[\hat h_{n_h}(\bar\by,t_{l-1})^\omega\mid\bar\by].
  \end{align*}
  On the proposal support, multiplying each weight by
  $P_{\theta,t_l}(\bar\by\mid\bar\bx)/Q^\tau_l(\bar\by\mid\bar\bx)$
  as in \eqref{eq:corrected-weight} corrects only this proposal mismatch, provided the proposal covers the base-kernel support.
  It does not correct the head-versus-rollout mismatch, completion tempering, nonlinear sharpening, or finite-$M$ self-normalization.
  The executed sampler omits this likelihood ratio; its uncorrected tempered update is not the exact $P_\theta$ transform.

  In particular, $\E[\hat h_{n_h}^{\omega}\mid\bar\by]$ generally differs from
  $(\E[\hat h_{n_h}\mid\bar\by])^\omega$ and from $h(\bar\by,t)^{\omega}$.
  Even with exact rollout values, raising $\E_P[\exp(R/\beta)\mid\bar\by]$ to $\omega$ is generally not the same as computing
  $\E_P[\exp(\omega R/\beta)\mid\bar\by]$.
  For $n_h=1$, sharpening does yield an expectation of $\exp(\omega R/\beta)$ under the direct completion law, but that law is still not the rollout law.
  A separate exact transform relative to a tempered reference chain is stated in \cref{app:tempered-reference}; it requires that chain's true downstream expectation, not the practical estimator.
  Finally, restricting guidance to a window omits exact transformed updates outside that window, and terminal best-of-$K$ ranking generally changes the output law even if each branch were sampled exactly.
  Thus these operating choices are not a guarantee of sampling \eqref{eq:tilted-target}; their empirical effects are studied in \cref{sec:ablation:tau_gamma}.
\end{remark}

\begin{remark}[Comparison with the continuous case]
  The continuous-time h-function identity in~\cite{zhu2026trainingfreeadaptationdiffusionmodels} differentiates the Gaussian transition density to bypass the requirement that $\phi$ be differentiable.
  The discrete construction avoids differentiating through $\phi$ for a different reason: we query $\phi$ only on a sampled clean completion $\hat\bx_0 \sim p^{\mathrm{head}}_{\theta,\gamma}(\cdot \mid \bar \by_{t_{l-1}})$, so a single forward pass through the prediction head suffices to produce $\hat h(\bar \by_{t_{l-1}},t_{l-1}) = \phi(\hat\bx_0)$ as in \eqref{eq:h-hat-candidate}.
  The trade-off is that this single-shot
  estimator uses the factorized prediction head \eqref{eq:factorized-head}
  instead of explicitly simulating the multi-step rollout marginal under the
  frozen reverse chain.
\end{remark}

\FloatBarrier
\section{Standard Masked Discrete Diffusion Sampling}
\label{sec:mdlm-baseline}

\Cref{alg:mdlm-standard} records the standard generative process of the pre-trained masked discrete diffusion model~\cite{sahoo2024simple} that D-DOIT builds on, provided here as background.
The forward process corrupts tokens with an absorbing $\mathsf m$ state.
The reverse process evaluates the clean-sequence prediction head $p^{\mathrm{head}}_\theta(\bx_0 \mid \bar \bx_{t_l})$ at each step.
It then unmasks each still-masked position with the per-step reveal schedule from \eqref{eq:mdlm-transition} and treats positions as conditionally independent.
Once a step reveals a token, later steps freeze it.
For consistency with the main text, \cref{alg:mdlm-standard} writes a linear
schedule with $\alpha_{t_l}=1-l/L$; other monotone schedules can be
substituted.

\begin{algorithm}[H]
  \caption{Standard Masked Discrete Diffusion Sampling (Baseline)}
  \label{alg:mdlm-standard}
  \begin{algorithmic}[1]
    \REQUIRE Pre-trained masked discrete diffusion model $p_\theta$, steps $L$
    \STATE Initialize $\bar \bx_{t_L} \gets (\mathsf m, \ldots, \mathsf m)$ \hfill\COMMENT{all positions masked}
    \FOR{$l = L, L-1, \ldots, 1$}
    \STATE Set $\alpha_{t_l}=1-l/L$,\; $\alpha_{t_{l-1}}=1-(l{-}1)/L$ \hfill\COMMENT{linear schedule for this baseline}
    \STATE Forward pass: obtain $\mu_\theta^j(v \mid \bar \bx_{t_l},t_l)$ for all positions $j$ \hfill\COMMENT{one model forward pass}
    \FOR{each masked position $j$ (where $\bar x_{t_l}^j = \mathsf m$)}
    \STATE Sample $\bar x_{t_{l-1}}^j$: unmask to token $v \in \calV$ with prob $\mu_\theta^j(v \mid \bar \bx_{t_l},t_l)(\alpha_{t_{l-1}}-\alpha_{t_l})/(1-\alpha_{t_l})$;\; stay masked with prob $(1-\alpha_{t_{l-1}})/(1-\alpha_{t_l})$ \hfill\COMMENT{\eqref{eq:mdlm-transition}}
    \ENDFOR
    \STATE Copy flag: $\bar x_{t_{l-1}}^j \gets \bar x_{t_l}^j$ for all unmasked $j$ \hfill\COMMENT{revealed tokens are frozen}
    \ENDFOR
    \STATE \textbf{return} $\bar \bx_0$
  \end{algorithmic}
\end{algorithm}

\par

\FloatBarrier
\section{Masked Diffusion Transition Details}
\label{app:masked-diffusion-details}

This section records the explicit matrices behind the forward and reverse
kernels summarized in \cref{sec:background-masked-discrete-diffusion}.
For $x\in\widetilde{\calV}$, let $e_x\in\mathbb R^{d+1}$ be its one-hot
vector including $e_{\mathsf m}$ for the mask token, and let
$\mathbf 1\in\mathbb R^{d+1}$ be the all-one vector.
Let $(X_{t_l})_{l=0}^L$ be the single-token chain with $X_0\in\calV$.
Under the column-vector convention, the forward transition matrix on
$[t_{l-1},t_l]$ is
\begin{align*}
  K_l
  =
  (1-b_l)I
  +
  b_l e_{\mathsf m}\mathbf 1^\top, \quad l=1,\ldots,L,
\end{align*}
so the transition matrix from time $0$ to $t_l$ is
$\bar K_l= K_l K_{l-1}\cdots K_1=
\alpha_l I
+
(1-\alpha_l)e_{\mathsf m}\mathbf 1^\top,$
for $\alpha_l=\prod_{r=1}^{l}(1-b_r)$.
The conditional one-step time reversal from $t_l$ to $t_{l-1}$ has the closed
form
\begin{align*}
  K_l^{\leftarrow,x_0}
  &=
  I
  +
  (\alpha_{l-1}-\alpha_l)/(1-\alpha_l) \cdot
  (e_{x_0}-e_{\mathsf m})e_{\mathsf m}^\top,
  \qquad
  l=1,\ldots,L,
\end{align*}
with kernel
$q_{t_l,t_{l-1}}(y\mid x,x_0)=\mathbb{P}(X_{t_{l-1}}=y\mid X_{t_l}=x,X_0=x_0)
= e_y^\top K_l^{\leftarrow,x_0}e_x$.
Equivalently,
\begin{align*}
  q_{t_l,t_{l-1}}(y\mid x,x_0)
  &=
  \begin{cases}
    (\alpha_{l-1}-\alpha_l)/(1-\alpha_l),
    & x=\mathsf m,\ y=x_0,\\
    (1-\alpha_{l-1})/(1-\alpha_l),
    & x=\mathsf m,\ y=\mathsf m,\\
    \mathbf 1\{y=x\},
    & x\neq\mathsf m,
  \end{cases}
\end{align*}
so an already unmasked token remains fixed while a masked token is either kept
masked or revealed as its original clean value.
Conditioned on the clean sequence, the true one-step reverse kernel
factorizes across coordinates,
$q_{t_l,t_{l-1}}(\bx_{t_{l-1}}\mid \bx_{t_l},\bx_0)
=\prod_{j=1}^N q_{t_l,t_{l-1}}(x_{t_{l-1}}^j\mid x_{t_l}^j,x_0^j)$,
which is the analytic posterior that \eqref{eq:mdlm-transition} approximates.
The same derivation covers the reverse kernel used by the standard masked
diffusion sampler in \cref{sec:mdlm-baseline}.

\FloatBarrier
\section{Per-output NFE Breakdown}
\label{app:nfe-breakdown}

We record the per-output number of function evaluations (NFE), where one NFE counts one sequence-level masked discrete diffusion forward pass, not one vectorized invocation on a batch.
This remains a model-only accounting convention: reward evaluations on completed sequences are reported separately and are not included in NFE.
Batching can combine many sequence-level evaluations into one physical call without reducing this count; actual latency also depends on hardware, batch size, memory use, and reward-model latency.
Let $|\cdot|$ denote set cardinality, $s_\spl$ the number of reverse updates before the best-of-$K$ split, and $[l_\eend,l^*]$ the guidance window from \cref{alg:doit-discrete}.
For update index $s \in \{0,\ldots,L-1\}$, the corresponding reverse state index is $l = L-s$.
Define $\calG = \{s : l_\eend \leq L-s \leq l^*\}$ as the set of guided updates, with $s\in\{0,\ldots,L-1\}$ throughout.
Let the active branch multiplicity be $b_s=1$ for $s<s_\spl$ and $b_s=K$ for $s\geq s_\spl$, and set $S=\sum_{s\in\calG}b_s$.
Under the stated one-head-pass-per-candidate implementation, the per-output model NFE is:
\begin{align}
  \label{eq:nfe-cost}
  \NFE
  &=
  |\{s : s < s_\spl\}|
  +K\cdot|\{s : s_\spl \leq s < L\}| \nonumber\\
  &\quad +M\cdot|\{s : s \in \calG,\, s < s_\spl\}|
  +K\cdot M\cdot|\{s : s \in \calG,\, s_\spl \leq s < L\}|.
\end{align}
The first line counts the shared base masked discrete diffusion forward pass on $\bar \bx_{t_l}$ at every update (one pass per update before the split, $K$ passes per update after).
The second line counts the $M$ candidate prediction-head passes for $h$-estimation inside the guidance window, also with the $K\times$ multiplier after the split.
Equivalently, $\NFE=\sum_{s=0}^{L-1}b_s+MS$.
Drawing $n_h$ completions per candidate leaves this model NFE unchanged when all completions reuse that candidate's prediction-head probabilities; the executed single-completion sampler uses $n_h=1$.
The corresponding number of scored candidate completions per output is
\begin{align*}
  \text{candidate reward evaluations}=n_h M S.
\end{align*}
If $K>1$ and a split occurs, terminal ranking requires at most $K$ additional reward evaluations per output when scores of the final sequences are not cached; candidate-completion scores can be reused only when they score those same final sequences.
The number of physical batched reward calls is implementation-dependent and is not this sequence-level evaluation count.
Setting $K = 1$ or $s_\spl = L$ disables the split and its terminal ranking, recovering the core D-DOIT branch.
These counts describe computational work rather than a wall-clock speed guarantee.

\FloatBarrier
\section{Detailed Sampler Pseudocode}
\label{app:guided-step}

\Cref{alg:doit-discrete} gives the outer sampling loop and \cref{alg:doit-guided-step} the per-step guided update.

\par\vspace{6pt}
\par\begin{minipage}{\linewidth}
\setlength{\parskip}{0pt}
\setlength{\parindent}{0pt}
\refstepcounter{algorithm}\label{alg:doit-discrete}
\hrule height 0.8pt
\par\vspace{4pt}
\textbf{Algorithm~\thealgorithm} D-DOIT with optional best-of-$K$ refinement. The guided step is defined in \cref{alg:doit-guided-step};
setting $K{=}1$ or $s_\mathrm{split}{=}L$ disables the split. Copies share all tokens revealed before the split; final reward selection changes the output law.
\par\vspace{4pt}
\hrule height 0.4pt
\par\vspace{4pt}
\begin{algorithmic}[1]
  \setlength{\itemsep}{1.5pt}
  \REQUIRE Frozen $P_\theta$, reward $R$, batch size $B$, steps $L$, guidance window $[l_\mathrm{end},l^*]$, split-update index $s_\mathrm{split}$, width $K$, guided-step parameters $(M,\omega,\beta,\tau,\gamma)$
  \STATE Initialize $B$ all-$\mathsf m$ trajectories at time $t_L$.
  \FOR{$s=0,1,\ldots,L-1$}
  \STATE If $s=s_\mathrm{split}$ and $K>1$, duplicate each active trajectory into $K$ copies.
  \STATE Set the reverse state index $l=L-s$.
  \STATE Apply \textsc{DoobGuidedStep} to active trajectories when $l_\mathrm{end}\leq l\leq l^*$; otherwise apply the standard transition \eqref{eq:mdlm-transition}.
  \ENDFOR
  \STATE If $K>1$ and $s_\mathrm{split}<L$, select the highest-reward clean copy in each group by \eqref{eq:bok-select}; otherwise return the generated clean sequences.
\end{algorithmic}
\hrule height 0.4pt
\end{minipage}\par

\par\vspace{12pt}
\par\begin{minipage}{\linewidth}
\setlength{\parskip}{0pt}
\setlength{\parindent}{0pt}
\refstepcounter{algorithm}\label{alg:doit-guided-step}
\hrule height 0.8pt
\par\vspace{4pt}
\textbf{Algorithm~\thealgorithm} \textsc{DoobGuidedStep}: the practical single-completion update, without proposal-ratio correction.
\par\vspace{4pt}
\hrule height 0.4pt
\par\vspace{4pt}
\begin{algorithmic}[1]
  \setlength{\itemsep}{1.5pt}
  \REQUIRE current state $\bar \bx_{t_l}$, step index $l$, frozen model $\theta$, reward $R$; $(M,\omega,\beta,\tau,\gamma)$
  \STATE Compute $\alpha_l$ and $\alpha_{l-1}$ from the noise schedule
  \hfill\COMMENT{schedule constants}
  \STATE Proposal pass: compute $\mu_\theta^j(\cdot\mid \bar \bx_{t_l},t_l)$ for all $j$
  \STATE Form the $\tau$-tempered proposal head
  $\mu_{\theta,\tau}^j(v\mid \bar \bx_{t_l},t_l)
  \propto
  \mu_\theta^j(v\mid \bar \bx_{t_l},t_l)^{1/\tau}$
  \STATE Draw $\{\bar \by_{t_{l-1}}^{(m)}\}_{m=1}^{M}$
  from \eqref{eq:mdlm-transition} with $\mu_\theta$ replaced by $\mu_{\theta,\tau}$
  \hfill\COMMENT{candidate reverse states}
  \FOR{$m = 1,\ldots,M$}
  \STATE Draw $\hat{\bx}_0^{(m)} \sim p^{\mathrm{head}}_{\theta,\gamma}(\cdot \mid \bar \by_{t_{l-1}}^{(m)})$
  via \eqref{eq:h-hat-candidate}
  \hfill\COMMENT{prediction-head completion}
  \STATE $\hat{h}(\bar \by_{t_{l-1}}^{(m)},t_{l-1})
  \gets \exp\bigl(R(\hat{\bx}_0^{(m)})/\beta\bigr)$
  \hfill\COMMENT{random head-proxy value}
  \ENDFOR
  \STATE Select $m^*$ by the weights in \eqref{eq:tilted-resample};
  \textbf{return} $\bar \bx_{t_{l-1}}=\bar \by_{t_{l-1}}^{(m^*)}$
\end{algorithmic}
\par\vspace{4pt}
\hrule height 0.4pt
\end{minipage}\par

\FloatBarrier
\section{Additional Analyses}
\label{app:analysis}
The main experiments evaluate HepG2 enhancer design, cell-type-specific
regulatory DNA design, and protein inverse folding. This appendix collects
approximation diagnostics and oracle-transfer analyses; the compute--quality
trade-off is discussed in the main text.
Matched-time sampling controls are detailed in
\cref{app:drakes-search-controls,app:ctrldna-search-controls}; full diagnostic
sweeps appear in \cref{sec:ablation}.

% \paragraph{Further Approximation and Search Diagnostics.}
% \label{sec:exp:diagnostics}

\paragraph{Concentrated but stochastic selection.}
At $M=8$ and $\omega=8$, mean ESS is $1.143\pm0.002$ and the top-weighted
candidate is selected on $95.72\%$ of guided steps. Soft resampling reaches
final guidance reward $5.237\pm0.051$, versus $5.154\pm0.027$ for direct
argmax in the same diagnostic (\cref{sec:ablation:argmax}). Thus selection
is near-greedy but not deterministic; low ESS measures local weight
concentration, not final-sequence diversity.

\paragraph{Suffix search and nonmonotone progress.}
With $M=9$, $K=4$, and $n_h=1$ fixed, moving the split from update $105$ to
$0$ increases final guidance reward from $5.272\pm0.057$ to
$5.890\pm0.028$, but raises wall time from $326.4\pm0.5$ to
$793.6\pm2.2$ seconds per job. Even with $90.38\%$ of bases fixed,
independent suffix completions differ at $66.1\%$ of the remaining positions.
Late branching therefore searches a diverse suffix conditional on a fixed
prefix; it cannot revise pre-split reveals. Separately, held-out activity
of argmax-completed partial states is nonmonotone despite higher final
activity than frozen sampling (\cref{sec:ablation:split}). These completion-proxy
scores are distinct from final guidance reward $R$ and the batch-median
Pred-Activity benchmark. The diagnostics characterize practical search
trade-offs, not exact reward-tilted sampling.

\FloatBarrier
\clearpage
\section{Experimental Details of Regulatory DNA Sequence Design}
\label{app:exp:drakes}

In this section, we provide the experimental details and additional experimental analysis for the regulatory DNA sequence design task in \cref{sec:exp:drakes}.

\subsection{Dataset and Training Details}
\label{app:exp:drakes:setup}

\paragraph{Dataset.}
We conduct experiments on the HepG2 enhancer corpus from \citet{gosai2023machine}.
The corpus contains ${\sim}700$k $200$\,bp human enhancer sequences with measured HepG2 activity.
We follow the dataset split of \citet{wang2025drakes} so that our evaluation seeds match the DRAKES test partition and our numbers compare against the published DRAKES table.

\paragraph{Reward oracle setup.}
We use two Enformer-based oracles~\cite{avsec2021effective} from \citet{wang2025drakes}.
The two oracles use disjoint chromosome splits of the same corpus.
The finetuning oracle provides the reward signal $R(\bx_0)$ during guided sampling.
The held-out evaluation oracle scores Pred-Activity at test time.
This matches the DRAKES protocol and makes our numbers comparable to the TDS, SMC, CG, and DRAKES results in \citet{wang2025drakes}.

\paragraph{Pre-training setup.}
D-DOIT operates on a frozen pre-trained masked discrete diffusion model~\cite{sahoo2024simple}.
We do not update its parameters during sampling.
We use the public masked discrete diffusion enhancer checkpoint of \citet{wang2025drakes}, which the authors pre-train on the ${\sim}700$k enhancer corpus.

\paragraph{Sampling setup.}
We use the two operating points defined in \cref{sec:operating-points}: D-DOIT-LBOK-S with a small per-step budget plus a late-stage trajectory split, and D-DOIT-Q with a larger inference-time budget for quality.
\Cref{tab:hyperparams} lists the exact hyperparameter values per benchmark.
All main DRAKES rows use a single A100 80GB GPU and follow the DRAKES evaluation protocol with $640$ sequences per seed.
Most rows use $3$ random seeds; the re-measured TDS and the adopted D-DOIT-LBOK-S rows use $5$ fresh final-test seeds after held-out eval selection.
The TDS and D-DOIT-LBOK-S rows are re-measured on an A100-SXM4-80GB GPU;
the remaining rows are reported as A100 reruns. Seed identity alone does not
guarantee bitwise agreement across hardware or software implementations, so
comparisons also require the same checkpoint, sampler, and evaluation protocol.
We do not copy numbers from \cite{wang2025drakes}, so wall-clock and reward differences against the published table reflect our hardware and seed protocol rather than reporting drift.

% BEGIN INLINE tables/hyper_discrete_doit
\begin{table}[t]
  \centering
  \caption{\textbf{Hyperparameters of D-DOIT.}
    We report two operating points per benchmark: \textbf{Q} = quality-oriented and \textbf{S} = speed-oriented.
  Q denotes high-budget quality; S denotes LBOK-S. Selection and test
protocols are described in \cref{tab:results_drakes,tab:results_ctrldna};
diagnostic sweeps specify their own configurations.}
  \label{tab:hyperparams}
  \small
  \setlength{\tabcolsep}{3.5pt}
  \begin{tabular}{@{}ll>{\raggedright\arraybackslash}p{5.1cm}lcccc@{}}
    \toprule
    & & & & \multicolumn{2}{c}{\textbf{DRAKES}} & \multicolumn{2}{c}{\textbf{Ctrl-DNA}} \\
    \cmidrule(lr){5-6} \cmidrule(lr){7-8}
    \textbf{Sym.} & \textbf{Name} & \textbf{Role} & \textbf{Range} & \textbf{Q} & \textbf{S} & \textbf{Q} & \textbf{S} \\
    \midrule
    $L$              & Steps     & Total discrete reverse steps.                                                                                            & $\{32,64,128\}$       & $128$ & $128$ & $128$ & $128$ \\
    $M$              & Budget    & Candidate next states per guided step; linear head-evaluation cost.                                      & $\{2,\dots,64\}$      & $52$  & $9$   & $52$  & $9$ \\
    $\omega$         & Strength  & Exponent on $\hat{h}$ in \eqref{eq:tilted-resample}; for one completion, log-weight sharpness is $\omega/\beta$.          & $[0.5,\,8.0]$         & $3.0$ & $3.0$ & $8.0$ & $3.0$ \\
    $\beta$          & Temp.     & Scales reward: $\hat{h} = \exp(R/\beta)$ in \eqref{eq:h-hat-candidate}. Small $\beta$ amplifies differences.              & $[0.1,\,2.0]$         & $0.2$ & $0.2$ & $0.5$ & $0.2$ \\
    $\tau$           & Prop.\ T  & Proposal temperature \eqref{eq:tempered-proposal} for candidate diversity; changes the reference kernel.                       & $[1.0,\,3.0]$         & $3.0$ & $3.0$ & $3.0$ & $3.0$ \\
    $\gamma$         & Comp.\ T  & Completion temperature for prediction-head completion $\hat{\bx}_0$ in \eqref{eq:h-hat-candidate}; $\gamma{<}1$ sharpens signal. & $[0.5,\,2.0]$         & $1.2$ & $1.2$ & $1.2$ & $1.2$ \\
    $l^*$            & Guided high & Highest reverse index where guidance is active; guided steps satisfy $l_\mathrm{end}\leq l\leq l^*$.                  & $\{L/4, \dots, L\}$   & $L$   & $L$   & $L$   & $L$ \\
    $l_\mathrm{end}$ & Guided low  & Lowest reverse index where guidance remains active; guided steps satisfy $l_\mathrm{end}\leq l\leq l^*$.              & $\{1,4,8\}$           & $1$   & $1$   & $1$   & $1$ \\
    $s_\mathrm{split}$ & Split idx. & Number of reverse updates before each trajectory splits into $K$ copies; $0$ applies best-of-$K$ from the start. & $\{0,0.5L,\dots,L\}$  & $0$   & $105$ & $64$  & $105$ \\
    $K$              & Width     & Number of copies per split trajectory \eqref{eq:bok-select}.                                                             & $\{1,2,4,8,20\}$      & $8$   & $4$   & $20$  & $4$ \\
    $B$              & Batch     & Physical output batch; affects memory use and measured throughput.                                                                    & ---                   & $32$  & $64$  & $256$ & $256$ \\
    $n_h$ & Completions & Independent head completions per candidate; distinct from sequence length $N$. & $\{1,2,3\}$ & $1$ & $1$ & $1$ & $1$ \\
    \bottomrule
  \end{tabular}
\end{table}
% END INLINE tables/hyper_discrete_doit

\subsection{Evaluation Metrics and Baselines}
\label{app:exp:drakes:metrics}

\paragraph{Evaluation metrics.}
We use four metrics that cover predicted activity, chromatin accessibility, and sequence naturalness:
\begin{itemize}[leftmargin=*, nosep]
  \item \textbf{Predicted activity (Pred-Activity):} We use the held-out evaluation reward oracle of \cite{wang2025drakes} to estimate HepG2 enhancer activity for each generated sequence and report the median across the batch following the DRAKES convention. Higher scores indicate more functional sequences.
  \item \textbf{Chromatin accessibility (ATAC-Acc):} We use an independent binary classifier trained on ENCODE HepG2 chromatin accessibility data~\cite{encode2012integrated} to test whether the designed sequences fall in accessible chromatin regions. This is a key property of active enhancers. Higher accuracy indicates more active sequences.
  \item \textbf{3-mer pearson correlation (3-mer Corr):} We compute the Pearson correlation between the $3$-mer frequency profile of the generated sequences and that of natural HepG2 enhancers from \cite{gosai2023machine}. Higher correlation indicates that the generated sequences match the overall distribution of natural enhancer sequences.
  \item \textbf{JASPAR motif correlation (JASPAR Corr):} We compute the Spearman correlation between the transcription-factor binding motif frequencies of the generated sequences and the natural enhancer reference, scored against the JASPAR motif database~\cite{castro2022jaspar}. Higher correlation indicates better preservation of regulatory motif structure.
\end{itemize}

\paragraph{Baselines.}
We compare D-DOIT against training-free and training-based methods.
Training-free methods steer a frozen pre-trained model at inference time:
\begin{itemize}[leftmargin=*, nosep]
  \item \textbf{Pre-trained:} the unguided masked discrete diffusion model of \cite{wang2025drakes}; provides the unguided reward and naturalness reference.
  \item \textbf{CG}~\cite{nisonoff2024unlocking}: classifier guidance with discrete gradients; backpropagates through a Gumbel-Softmax relaxation and requires a differentiable reward.
  \item \textbf{SMC} and \textbf{TDS}~\cite{wu2024practical}: particle-based guidance with path likelihood ratios. TDS is a strong reported training-free baseline and provides the runtime
reference for LBOK-S; its budget is not matched to Q.
  \item \textbf{SVDD-PM}~\citep{li2025derivativefree}: posterior-mean value decoding, using the same frozen MDLM and oracle protocol, with $M=416$ for similar runtime to Q.
  \item \textbf{Frozen Best-of-$N_{\mathrm{BoN}}$:} endpoint-only selection over independent unguided samples; $N_{\mathrm{BoN}}=732$ in the Q wall-clock comparison.
\end{itemize}
Training-based methods update model parameters:
\begin{itemize}[leftmargin=*, nosep]
  \item \textbf{DRAKES}~\cite{wang2025drakes}: reward finetunes the pre-trained masked discrete diffusion model with a KL regularizer; we treat it as the strongest naturalness-preserving training-based reference.
  \item \textbf{DRAKES w/o KL}: removes the KL term to expose the over-optimization risk of unregularized reward maximization.
  \item \textbf{CFG}~\cite{ho2022classifier}: trains a conditional model from scratch.
\end{itemize}

\subsection{Nearest-neighbor and Endpoint-search Controls}
\label{app:drakes-search-controls}
The SVDD-PM comparison uses the same frozen MDLM, guidance Enformer,
held-out evaluator, paired seeds $100$--$102$, and $640$ final outputs per
method and seed. We calibrate its candidate count to $M=416$ to give a
similar wall-clock budget to Q ($L=128,M=52,K=8,s_\spl=0$).
Pred-Activity is computed as a median within each seed, then summarized by
the mean and sample standard deviation across seeds. The measured times
are $10.595$ and $11.236$ seconds per output for Q and SVDD-PM, respectively.
Q has higher means on all four metrics; the clearest separation is ATAC-Acc
($91.67\%$ versus $39.84\%$), while Pred-Activity differs by $0.15$.

% BEGIN INLINE tables/drakes_search_controls
\begin{table}[!ht]
  \centering
  \caption{\textbf{Matched-time controls.}
  We test D-DOIT-Q vs SVDD-PM and Frozen Best-of-732 at similar measured wall-clock
  cost. Bold
  marks the best mean per column.}
  \label{tab:drakes-search-controls}
  \small
  \setlength{\tabcolsep}{6pt}
  \renewcommand{\arraystretch}{1.1}
  \begin{tabular}{@{}lccc@{}}
    \toprule
    Metric & D-DOIT-Q & SVDD-PM & BoN \\
    \midrule
    s/output $\downarrow$ & $10.595$ & $11.236$ & $\mathbf{10.587}$ \\
    Activity $\uparrow$ & $\mathbf{\MainResult\pm\MainResultUncertainty}$ & $6.04\pm0.07$ & $4.70\pm0.03$ \\
    ATAC (\%) $\uparrow$ & $\mathbf{91.7\pm0.9}$ & $39.8\pm0.4$ & $43.5\pm1.3$ \\
    3-mer $\uparrow$ & $\mathbf{0.894\pm0.009}$ & $0.688\pm0.012$ & $0.881\pm0.008$ \\
    JASPAR $\uparrow$ & $0.923\pm0.002$ & $0.917\pm0.002$ & $\mathbf{0.957\pm0.002}$ \\
    \bottomrule
  \end{tabular}
\end{table}
% END INLINE tables/drakes_search_controls

\noindent Frozen Best-of-$N_{\mathrm{BoN}}$ samples independent complete sequences from
the same unguided backbone, scores them with the guidance Enformer, and
retains the highest-scoring endpoint in each group. A separate held-out
Enformer evaluates the retained sequences. Same-A100 calibration gives
$N_{\mathrm{BoN}}=732$, with $10.587$ seconds per output compared with
Q's $10.595$. \Cref{tab:drakes-search-controls} reports this high-budget
comparison independently of the lower-budget diagnostics; Q improves
activity, accessibility, and 3-mer correlation, while endpoint selection
achieves higher JASPAR correlation ($0.957$ versus $0.923$). This control
tests whether spending the same time on more complete unguided endpoints can
replace local value-guided transitions. Endpoint selection does not reproduce
the activity and accessibility gains, although it better preserves this motif
statistic. Reward evaluation is included in the measured wall-clock cost.

\paragraph{Separating model and reward budgets.}
\Cref{tab:reward-call-accounting} additionally compares diagnostic settings
with five seeds and $640$ outputs per seed. The branched $M=9$ setting uses
$1{,}970$ diffusion evaluations per output, close to Best-of-15's $1{,}920$,
but incurs substantially more oracle work. Its $1{,}777$ reward evaluations
include $1{,}773$ candidate completions and four terminal scores. Across a
$640$-output job, batching combines $1{,}137{,}280$ scored sequences into
$1{,}290$ vectorized oracle calls. These counts explain why matching only
backbone evaluations does not match wall-clock time.

% BEGIN INLINE tables/reward_call_accounting
\begin{table}[t]
  \centering
  \caption{\textbf{Backbone work, oracle work, and measured time on DRAKES.}
  Each setting uses five sampling seeds and $640$ outputs per seed.
  Pred-Activity is held-out Enformer activity; entries report means and sample
  standard deviations across seeds. The final column measures the full
  $640$-output job and must not be extrapolated to other physical batch sizes.
  The unbranched and branched rows are diagnostic settings, not substitutions
  for the separately selected main-table LBOK-S or Q results.}
  \label{tab:reward-call-accounting}
  \footnotesize
  \setlength{\tabcolsep}{4pt}
  \begin{tabular}{@{}lcccc@{}}
    \toprule
    Configuration & Pred-Activity $\uparrow$ & \makecell{Model NFE\\per output} & \makecell{Reward evaluations\\per output} & s/640 outputs \\
    \midrule
    Frozen Best-of-15 & $2.219\pm0.017$ & $1{,}920$ & $15$ & $97.8\pm1.5$ \\
    $M=9,\tau=1,K=1$ & $3.115\pm0.021$ & $1{,}280$ & $1{,}152$ & $227.0\pm1.9$ \\
    $M=9,\tau=3,K=1$ & $3.670\pm0.072$ & $1{,}280$ & $1{,}152$ & $227.8\pm2.6$ \\
    $M=9,\tau=3,K=4,s_\spl=105$ & $4.491\pm0.045$ & $1{,}970$ & $1{,}777$ & $327.8\pm4.3$ \\
    \bottomrule
  \end{tabular}
\end{table}
% END INLINE tables/reward_call_accounting

\subsection{Guidance-oracle and Evaluator Robustness}
\label{app:drakes-oracles}
We distinguish two controls, both using five sampling seeds and $640$
sequences per seed. In \cref{tab:oracle-guidance}, only the guidance oracle
changes among Enformer, CNN, and MLP; sampler settings, seeds, output count,
held-out Enformer, and biological metrics remain fixed. All three guidance
oracles improve held-out activity and accessibility over the frozen model,
but the effect size varies with the oracle. Enformer guidance gives the
highest activity and accessibility, whereas MLP guidance gives the highest
3-mer correlation in this diagnostic.

In \cref{tab:oracle-evaluators}, we instead rescore the same Enformer-guided
and unguided sequences with held-out Enformer, CNN, and MLP predictors.
D-DOIT exceeds the frozen model within each evaluator. Because their output
scales differ, cross-column score magnitudes are not directly comparable.
These controls assess predictor transfer separately from the selected
high-budget Q row and do not substitute for experimental biological assays.

% BEGIN INLINE tables/oracle_guidance
\begin{table}[t]
  \centering
  \caption{\textbf{Changing the guidance oracle.}
    Only the guidance reward $R$ changes among the D-DOIT rows; the sampler,
    seeds, output count, held-out Enformer evaluator, and other metrics are
    fixed. Entries are mean $\pm$ sample standard deviation over five seeds,
    with $640$ sequences per seed. Pred-Activity uses the held-out Enformer.
    These results are separate from the high-budget Q evaluation.}
  \label{tab:oracle-guidance}
  \small
  \setlength{\tabcolsep}{4pt}
  \resizebox{\linewidth}{!}{%
    \begin{tabular}{@{}lcccc@{}}
      \toprule
      Generation method & Pred-Activity $\uparrow$ & ATAC-Acc (\%) $\uparrow$ & 3-mer Corr $\uparrow$ & JASPAR Corr $\uparrow$ \\
      \midrule
      Enformer-guided D-DOIT & $4.491\pm0.045$ & $68.25\pm1.86$ & $0.8630\pm0.0160$ & $0.9468\pm0.0025$ \\
      CNN-guided D-DOIT & $2.203\pm0.054$ & $51.91\pm1.19$ & $0.8738\pm0.0013$ & $0.9342\pm0.0021$ \\
      MLP-guided D-DOIT & $1.544\pm0.037$ & $28.78\pm1.29$ & $0.9489\pm0.0035$ & $0.9347\pm0.0031$ \\
      Frozen pre-trained model & $0.371\pm0.018$ & $1.66\pm0.59$ & $-0.0584\pm0.0149$ & $0.2723\pm0.0220$ \\
      \bottomrule
    \end{tabular}%
  }
\end{table}
% END INLINE tables/oracle_guidance
% BEGIN INLINE tables/oracle_evaluators
\begin{table}[t]
  \centering
  \caption{\textbf{Transfer across held-out evaluators.}
    The same Enformer-guided D-DOIT and frozen-model sequences are rescored
    by three held-out predictors. Entries are mean $\pm$ sample standard
    deviation over five seeds, with $640$ sequences per seed. Enformer scores
    are Pred-Activity; CNN and MLP scores have their own scales. Compare
    methods within each column, not absolute values across columns.}
  \label{tab:oracle-evaluators}
  \small
  \setlength{\tabcolsep}{4pt}
  \resizebox{\linewidth}{!}{%
    \begin{tabular}{@{}lccc@{}}
      \toprule
      Generation method & \makecell{Pred-Activity $\uparrow$\\(held-out Enformer)} & Held-out CNN $\uparrow$ & Held-out MLP $\uparrow$ \\
      \midrule
      Enformer-guided D-DOIT & $4.491\pm0.045$ & $0.867\pm0.014$ & $0.634\pm0.015$ \\
      Frozen pre-trained model & $0.371\pm0.018$ & $0.484\pm0.014$ & $0.509\pm0.014$ \\
      \bottomrule
    \end{tabular}%
  }
\end{table}
% END INLINE tables/oracle_evaluators

\subsection{Additional Experimental Analysis}
\label{app:exp:drakes:analysis}

For the experimental results in \cref{tab:results_drakes}, we provide the additional detailed analysis below.

\Cref{sec:exp:drakes} reports the headline TDS comparisons.
We add two pieces of detail here.
First, the D-DOIT-LBOK-S vs.\ TDS gap on ATAC-Acc is $+25.0$ percentage points ($40.0\% \to 65.0\%$), while Pred-Activity is matched ($4.56$ vs.\ $4.58$) with lower reported seed variability; this is descriptive, not a variance-test result.
The same operating point also lifts ATAC-Acc from $39.9\%$ to $65.0\%$ over SMC and JASPAR Corr from $0.756$ to $0.948$ over SMC.
Second, high-budget D-DOIT-Q over TDS gives $+1.61$ in Pred-Activity, $+51.7$ percentage points in ATAC-Acc, $+0.077$ in 3-mer Corr, and $+0.124$ in JASPAR Corr.

Compared with the training-based DRAKES, D-DOIT-Q achieves higher Pred-Activity ($\MainResult$ vs.\ $5.61$), higher 3-mer Corr ($0.894$ vs.\ $0.887$), and higher JASPAR Corr ($0.923$ vs.\ $0.911$) without any backbone update.
DRAKES remains higher on ATAC-Acc ($92.5\%$ vs.\ $91.7\%$), so we do not claim an all-metric win.
The unregularized variant DRAKES w/o KL exposes the cost of training-based reward maximization: it gains $+0.83$ in Pred-Activity over DRAKES but loses $0.580$ in 3-mer Corr ($0.887 \to 0.307$) and $0.354$ in JASPAR Corr ($0.911 \to 0.557$).
D-DOIT-Q stays near the DRAKES high-naturalness regime.
The frozen model supplies candidate proposals without retraining. The exact
reference transform reweights its induced distribution; the practical
sampler's temperatures and selection operations change that reference law.
\Cref{fig:pareto_front} plots utility and naturalness for the original
benchmark comparison set. The separately time-calibrated SVDD and endpoint
controls are reported in \cref{tab:drakes-search-controls}, rather than
pooled into a claim of an exhaustive compute--quality frontier.
D-DOIT-LBOK-S anchors the practical moderate-budget side of this frontier.
D-DOIT-Q occupies the high-utility, high-naturalness corner when inference-time budget is deliberately relaxed, with ATAC-Acc as the remaining gap to DRAKES.

\subsection{Held-out Final-test Check for D-DOIT-LBOK-S}
\label{app:drakes-lbok-heldout}

We select the DRAKES LBOK-S configuration on eval seeds $42,43,44$ and report it once on fresh final-test seeds $105,106,107,108,109$.
The fixed final configuration uses $M{=}9$, $K{=}4$, $s_\mathrm{split}{=}105$, $\omega{=}3.0$, $\beta{=}0.2$, $\tau{=}3.0$, and $\gamma{=}1.2$.
It exceeds the TDS thresholds on ATAC-Acc, 3-mer Corr, and JASPAR Corr, and has a similar mean Pred-Activity to TDS; no statistical equivalence test is implied.

\begin{table}[ht]
  \centering
  \caption{\textbf{Held-out final-test check for D-DOIT-LBOK-S.}
    Targets are the TDS row in \cref{tab:results_drakes}.
  The labels compare reported means with the reference values; they are not
significance or statistical-equivalence decisions.}
  \label{tab:drakes-lbok-heldout-final}
  \small
  \begin{tabular}{@{}lccc@{}}
    \toprule
    Metric & TDS reference & D-DOIT-LBOK-S final test & Mean comparison \\
    \midrule
    Pred-Activity & $4.58$ & $4.560(0.040)$ & similar mean \\
    ATAC-Acc & $40.0\%$ & $65.03(1.43)\%$ & pass \\
    3-mer Corr & $0.817$ & $0.8511(0.0165)$ & pass \\
    JASPAR Corr & $0.798$ & $0.9477(0.0016)$ & pass \\
    \bottomrule
  \end{tabular}
\end{table}

\subsection{Held-out Selection for High-budget D-DOIT-Q}
\label{app:drakes-heldout}

To avoid tuning directly on the final reported seeds, we separate the DRAKES high-budget update into three seed groups.
Development seeds $0,1,2$ identify promising sampler families only.
Held-out eval seeds $42,43,44$ select the final high-budget D-DOIT-Q configuration ($M{=}52$, $K{=}8$, $s_\mathrm{split}{=}0$, $\omega{=}3.0$, $\beta{=}0.2$, $\tau{=}3.0$, $\gamma{=}1.2$).
Final-test seeds $100,101,102$ are then evaluated once and reported as the held-out final numbers in \cref{tab:results_drakes}.

\begin{table}[ht]
  \centering
  \caption{\textbf{Held-out final-test check for D-DOIT-Q.}
    Targets are the DRAKES row in Table~\ref{tab:results_drakes}. The final-test
    result passes Pred-Activity, 3-mer Corr, and JASPAR Corr, but misses the
  ATAC-Acc target on the held-out evaluation.}
  \label{tab:drakes-heldout-final}
  \small
  \begin{tabular}{@{}lccc@{}}
    \toprule
    Metric & DRAKES reference & D-DOIT-Q final test & Mean comparison \\
    \midrule
    Pred-Activity & $5.61$ & $\MainResult(\MainResultUncertainty)$ & pass \\
    ATAC-Acc & $92.5\%$ & $91.7(0.9)\%$ & miss \\
    3-mer Corr & $0.887$ & $0.894(0.009)$ & pass \\
    JASPAR Corr & $0.911$ & $0.923(0.002)$ & pass \\
    \bottomrule
  \end{tabular}
\end{table}

% BEGIN INLINE figures/fig_pareto_front
\begin{figure}[!ht]
  \centering
  \includegraphics[width=0.85\linewidth]{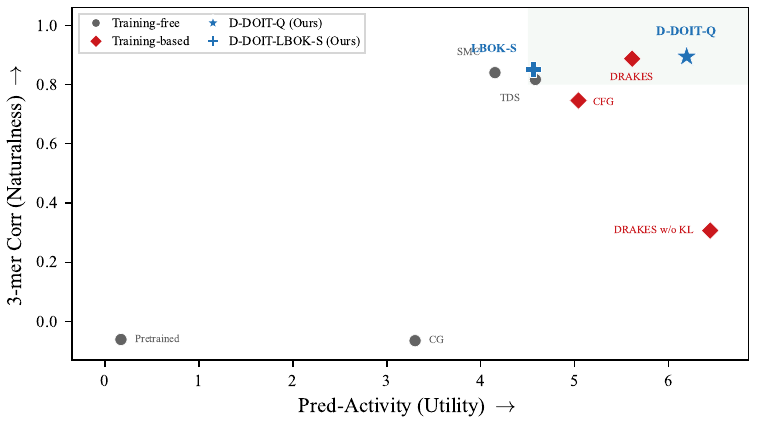}
  \caption{\textbf{Utility-naturalness trade-off.}
    We plot Pred-Activity (median) against 3-mer Corr; one point per method.
    Green shading marks the high-utility, high-naturalness region.
    This view covers the original benchmark rows; the separately time-calibrated
    SVDD and endpoint comparisons are in \cref{tab:drakes-search-controls}.
  ATAC-Acc is reported separately in \cref{tab:results_drakes,tab:drakes-heldout-final}.}
  \label{fig:pareto_front}
\end{figure}
% END INLINE figures/fig_pareto_front

\FloatBarrier
\clearpage
\section{Experimental Details of Cell-type-specific Enhancer Design}
\label{app:exp:ctrldna}

In this section, we provide the experimental details and additional experimental analysis for the cell-type-specific enhancer design task in \cref{sec:exp:ctrldna}.

\subsection{Dataset and Training Details}
\label{app:exp:ctrldna:setup}

\paragraph{Dataset.}
We use the $200$\,bp human enhancer corpus of \citet{chen2025ctrldna}, covering three cell lines: HepG2, K562, and SK-N-SH.
Data preprocessing follows the regLM protocol of \citet{lal2024reglm}, as in \citet{chen2025ctrldna}.
We use the Ctrl-DNA enhancer task only for its data, cell-type-specific reward oracles, and evaluation metrics.
The training procedure and backbone architecture come from our own setup, which we describe in the next paragraphs.

\paragraph{Reward oracle setup.}
We use the cell-type-specific Enformer-based oracles~\cite{avsec2021effective} from \citet{chen2025ctrldna}, which the authors train on the same corpus.
Unlike constrained-RL baselines with explicit thresholds and policy retraining, D-DOIT is an inference-time method and accepts a single black-box scalar reward.
For target cell line $c$ and off-target cell lines $\calC_{-c}$, we set
\begin{align}
  \label{eq:ctrldna-reward}
  R(\bx_0) = R_c(\bx_0) - \lG \cdot \max_{c' \in \calC_{-c}} R_{c'}(\bx_0),
\end{align}
where each $R_{c}$ is the cell-type-specific Enformer oracle of \citet{chen2025ctrldna}.
The $\max$ over off-targets suppresses the worst case during sampling.
The evaluation $\Delta R$ averages over off-targets to follow the Ctrl-DNA convention.
We sweep $\lG$ on held-out eval seeds disjoint from the final-test seeds.
The final high-budget row uses the task-default penalties selected for the Q operating point, and the final low-budget LBOK-S row uses $\lG=2.0$ selected by the low-budget eval sweep.

\paragraph{Pre-training setup.}
We train a new masked discrete diffusion model~\cite{sahoo2024simple} on the Ctrl-DNA enhancer corpus.
The HyenaDNA backbone used by \citet{chen2025ctrldna} is autoregressive and does not fit our discrete $h$-transform.
D-DOIT operates on this frozen model without updating its parameters during sampling. The same checkpoint serves all three cell lines and both operating points.

\paragraph{Sampling setup.}
We use the two operating points from \cref{sec:operating-points}.
\Cref{tab:hyperparams} lists the exact hyperparameter values per benchmark.
For the D-DOIT Ctrl-DNA rows, we choose fixed configs on eval seeds $42,43,44$ and report final-test seeds $100,101,102,103,104$.
We use $5$ random seeds with batch size $256$ per seed and run a separate set for each of the three target cell lines.
Pre-trained, CG, SMC, and TDS are the rerun training-free comparison rows.
Other baselines come from the reported results in \citet{chen2025ctrldna} Table~1 under the $5$-seed, $\delta = 0.5$ setting.

\subsection{Evaluation Metrics and Baselines}
\label{app:exp:ctrldna:metrics}

\paragraph{Evaluation metrics.}
We use three metrics that cover cell-type specificity, motif structure, and set-level diversity:
\begin{itemize}[leftmargin=*, nosep]
  \item \textbf{Cell-specific reward difference ($\Delta R$):} We compute the difference between the target-cell oracle score and the average of the off-target oracle scores following the Ctrl-DNA convention. Higher scores indicate stronger cell-type specificity.
  \item \textbf{Motif correlation (Motif Corr):} We compute the Pearson correlation between the transcription-factor binding-site frequencies of the generated sequences and the cell-type-specific natural enhancer set from \cite{chen2025ctrldna} (top half by target activity, bottom half by off-target activity). Higher correlation indicates that the generated sequences match the regulatory grammar of natural cell-specific enhancers.
  \item \textbf{Set-level diversity (Diversity):} We compute the Shannon entropy across the generated sequence set with the official Ctrl-DNA evaluation script~\cite{chen2025ctrldna}. Higher entropy indicates that the sampler does not collapse onto a small set of high-reward sequences.
\end{itemize}
$\Delta R$ and Motif Corr appear in \cref{tab:results_ctrldna}.
Diversity appears in \cref{fig:diversity_ctrldna}.
We omit Median because \cite{chen2025ctrldna} mentions it in text but does not report numerical values.

\paragraph{Baselines.}
We compare D-DOIT against the baselines from \citet{chen2025ctrldna} and the same training-free guidance methods as in \cref{app:exp:drakes:metrics}:
\begin{itemize}[leftmargin=*, nosep]
  \item \textbf{Pre-trained}, \textbf{CG}~\cite{nisonoff2024unlocking}, \textbf{SMC}, \textbf{TDS}~\cite{wu2024practical}: training-free inference-time methods on the same frozen masked discrete diffusion model. These share the frozen backbone with D-DOIT. Frozen Best-of-708 additionally
provides the same-A100 endpoint-search control at Q's measured time budget.
  \item \textbf{regLM}~\cite{lal2024reglm}: supervised generative baseline that the authors train on the cell-type-specific enhancer corpus.
  \item \textbf{PPO, TACO, PPO-Lagrangian}~\cite{chen2025ctrldna}: RL-based methods that retrain the generator against the specificity reward.
  \item \textbf{AdaLead, PEX, BO, CMA-ES}~\cite{chen2025ctrldna}: directed-evolution methods that search the sequence space directly without any generative backbone.
\end{itemize}

\subsection{Training-free Baseline Tuning Screen}
\label{app:exp:ctrldna:baseline-tuning}

To check that the Ctrl-DNA training-free baselines are not weak only because of arbitrary default hyperparameters, we run a bounded seed-42 screen for CG, SMC, and TDS on HepG2, K562, and SK-N-SH.
\Cref{tab:ctrldna_baseline_tuning} summarizes all 13 screened configurations at the family level and names the selected follow-up configuration for each baseline family.
Across this screen, CG remains a high-motif, near-zero-reward baseline; SMC obtains only small reward movement and degrades motif/diversity; and TDS remains weak on the composite specificity reward.
Because this is a one-seed screen rather than a final-test protocol, we use it as baseline-tuning provenance instead of replacing the main-table baseline rows reported above.

% BEGIN INLINE tables/ctrldna_baseline_tuning
\begin{table}[t]
  \centering
  \renewcommand{\arraystretch}{1.12}
  \setlength{\tabcolsep}{3.2pt}
  \caption{\textbf{Ctrl-DNA training-free baseline tuning screen.}
    We run a bounded seed-42 screen over 13 CG, SMC, and TDS hyperparameter settings on all three Ctrl-DNA cell tasks.
    The screen documents that these training-free baselines remain weak after tuning: CG preserves motifs without improving reward/specificity, SMC gains little reward while degrading motif/diversity, and TDS stays near zero reward movement.
  The selected configs are provenance for possible follow-up evaluation, not final-test rows.}
  \label{tab:ctrldna_baseline_tuning}
  \footnotesize
  \begin{tabular}{@{}l>{\raggedright\arraybackslash}p{2.75cm}>{\raggedright\arraybackslash}p{2.45cm}>{\raggedright\arraybackslash}p{3.0cm}>{\raggedright\arraybackslash}p{2.85cm}@{}}
    \toprule
    Method & Seed-42 screen & Selected config & Screen outcome & Takeaway \\
    \midrule
    CG  & guidance scale $\{20,50,100,200\}$ & CG\_G20: guidance scale $20$ & $\Delta R \in [-0.002,0.000]$; specificity $\le {-}0.007$; motif corr $0.903$--$0.916$ & All screened CG settings preserve motifs but do not improve reward or specificity. \\
    SMC & beta values: $0.3$, $0.5$, $0.75$, $1.0$, $1.5$ & SMC\_B05: beta $0.5$ & Best $\Delta R=0.011$ at beta $1.5$, with motif corr $0.388$ and diversity $0.928$; selected beta $0.5$ has $\Delta R=0.006$. & Small reward gains come with much weaker motif/diversity, so SMC is not competitive. \\
    TDS & $g$/beta pairs: $2/0.75$, $5/0.75$, $5/1.0$, $10/1.0$ & TDS\_G5\_B10: $g=5$, beta $1.0$ & $\Delta R \in [-0.002,0.001]$; specificity $<0$ for all settings; best motif corr $0.510$. & TDS remains weak after this bounded tuning pass. \\
    \bottomrule
  \end{tabular}
\end{table}
% END INLINE tables/ctrldna_baseline_tuning

\subsection{Additional Experimental Analysis}
\label{app:exp:ctrldna:analysis}

For the experimental results in \cref{tab:results_ctrldna}, we provide the additional detailed analysis below.

\Cref{tab:results_ctrldna} records positive mean specificity for the two
D-DOIT configurations relative to the originally reported CG, SMC, and TDS
rows. The time-calibrated endpoint control additionally attains positive
specificity, with mean $\Delta R=0.108$.
Compared with TDS, D-DOIT-LBOK-S improves $\Delta R$ from $0.002$ to $0.033$ and improves Motif Corr from $0.573$ to $0.648$, at about $0.94\times$ the rerun TDS wall-clock in the Ctrl-DNA timing logs.
CG fails to move the specificity reward at all ($\Delta R \approx 0$, indistinguishable from Pre-trained), and SMC sits at $\Delta R = 0.000$ with Motif Corr $0.587$.
Q reaches $\Delta R=0.119\pm0.063$ and Motif Corr $0.498\pm0.070$.
Its $5.43$ seconds per output is close to Frozen Best-of-708's $5.51$
on the same A100. The small mean advantage over endpoint selection is not
claimed statistically significant. \Cref{tab:results_ctrldna_per_cell}
provides the per-cell breakdown for the original guidance comparisons.

Compared with the remaining training-based baselines, D-DOIT-Q stays behind regLM on $\Delta R$ ($0.119$ vs.\ $0.16$).
Because regLM uses a different generative backbone and training procedure,
this gap does not isolate the effect of inference-time adaptation.
On Motif Corr, D-DOIT-Q at $0.498$ exceeds every remaining training-based and evolutionary baseline (PPO-Lagrangian $0.42$, AdaLead $0.36$).
These means describe motif preservation at the reported specificity levels;
they do not identify which mechanism caused the difference.
Pre-trained has high Motif Corr ($0.937$) at $\Delta R\approx0$, providing
an unguided naturalness reference rather than a mathematical ceiling.
LBOK-S has higher Motif Corr ($0.648$) than Q and Frozen Best-of-708, while
achieving lower specificity.
D-DOIT-Q ($0.498$) trades motif fidelity for higher $\Delta R$, so we use it as the high-specificity operating point rather than a motif-preserving one.

\Cref{fig:diversity_ctrldna} compares diversity and specificity within the
original guidance baseline set; the endpoint control's reward and motif
scores are compared separately in \cref{tab:results_ctrldna}.
SMC and TDS gain little $\Delta R$ and collapse Shannon entropy to $\approx 1.5$.
D-DOIT-Q occupies the high-$\Delta R$, high-diversity region across all three cell lines.
\Cref{tab:results_ctrldna_per_cell} reports the per-cell breakdown of $\Delta R$, Motif Corr, and Diversity.

% BEGIN INLINE figures/fig_diversity_ctrldna
\begin{figure*}[!ht]
  \centering
  \includegraphics[width=\textwidth]{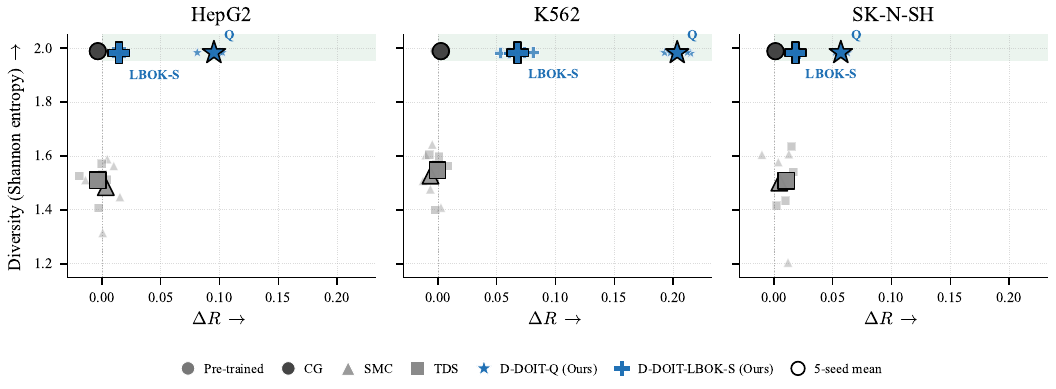}
  \caption{\textbf{Specificity vs.\ diversity on Ctrl-DNA.}
    Each panel is one target cell line.
    Small markers are single seeds, and large black-edged markers are the corresponding $5$-seed means.
    The $x$-axis is cell-specific reward gain $\Delta R$; the $y$-axis is set-level Shannon entropy from the Ctrl-DNA evaluation script~\cite{chen2025ctrldna}.
    The green region marks positive $\Delta R$ with high diversity.
  D-DOIT moves into this region, whereas Pre-trained/CG remain near zero reward movement and SMC/TDS lose diversity.}
  \label{fig:diversity_ctrldna}
\end{figure*}
% END INLINE figures/fig_diversity_ctrldna

The headline averages in \cref{tab:results_ctrldna} are the cross-cell means of these rows.
Both D-DOIT operating points achieve their largest $\Delta R$ on K562, consistent with K562 carrying the largest reward headroom in the Ctrl-DNA setup, with smaller positive gains on HepG2 and SK-N-SH.
D-DOIT-LBOK-S wins on Motif Corr for every cell line among training-free methods that move the reward.
Both D-DOIT operating points keep Diversity high ($\approx 1.98$) on all three cells.
SMC and TDS drop to $\approx 1.5$.

% BEGIN INLINE tables/results_ctrldna_per_cell
\begin{table}[t]
  \centering
  \small
  \renewcommand{\arraystretch}{1.2}
  \caption{\textbf{Per-cell-type Ctrl-DNA results.}
    We break down $\Delta R$, Motif Corr, and Diversity by target cell line.
    Parentheses denote standard deviation across $5$ random seeds.
  \textbf{Bold} = best, \underline{underline} = second best, per column within each cell line.}
  \label{tab:results_ctrldna_per_cell}
  \begin{tabular}{@{}llccc@{}}
    \toprule
    Cell line & Method & $\Delta R$\,$\uparrow$ & Motif Corr\,$\uparrow$ & Diversity\,$\uparrow$ \\
    \midrule
    \multirow{6}{*}{HepG2}
    & Pre-trained                     & $-0.004(0.002)$          & $\underline{0.937(0.009)}$ & $\mathbf{1.988(0.001)}$ \\
    & CG                             & $-0.004(0.002)$          & $\mathbf{0.938(0.009)}$    & $\mathbf{1.988(0.001)}$ \\
    & SMC                            & $+0.003(0.011)$          & $0.582(0.094)$             & $1.484(0.109)$ \\
    & TDS                            & $-0.004(0.009)$          & $0.606(0.087)$             & $1.509(0.061)$ \\
    & \textbf{D-DOIT-LBOK-S (Ours)}  & $\underline{+0.014(0.003)}$ & $0.667(0.032)$          & $1.981(0.002)$ \\
    & \textbf{D-DOIT-Q (Ours)}       & $\mathbf{+0.095(0.008)}$ & $0.558(0.048)$             & $\underline{1.982(0.001)}$ \\
    \midrule
    \multirow{6}{*}{K562}
    & Pre-trained                     & $+0.003(0.003)$          & $\underline{0.937(0.009)}$ & $\mathbf{1.988(0.001)}$ \\
    & CG                             & $+0.002(0.003)$          & $\mathbf{0.938(0.009)}$    & $\mathbf{1.988(0.001)}$ \\
    & SMC                            & $-0.006(0.006)$          & $0.613(0.133)$             & $1.527(0.095)$ \\
    & TDS                            & $-0.000(0.006)$          & $0.538(0.072)$             & $1.546(0.084)$ \\
    & \textbf{D-DOIT-LBOK-S (Ours)}  & $\underline{+0.068(0.010)}$ & $0.620(0.031)$          & $\underline{1.982(0.001)}$ \\
    & \textbf{D-DOIT-Q (Ours)}       & $\mathbf{+0.204(0.009)}$ & $0.421(0.037)$             & $\underline{1.982(0.001)}$ \\
    \midrule
    \multirow{6}{*}{SK-N-SH}
    & Pre-trained                     & $+0.001(0.001)$          & $\underline{0.937(0.009)}$ & $\mathbf{1.988(0.001)}$ \\
    & CG                             & $+0.001(0.001)$          & $\mathbf{0.938(0.010)}$    & $\mathbf{1.988(0.001)}$ \\
    & SMC                            & $+0.004(0.009)$          & $0.566(0.133)$             & $1.501(0.170)$ \\
    & TDS                            & $+0.010(0.006)$          & $0.573(0.075)$             & $1.508(0.088)$ \\
    & \textbf{D-DOIT-LBOK-S (Ours)}  & $\underline{+0.018(0.002)}$ & $0.658(0.054)$          & $1.982(0.001)$ \\
    & \textbf{D-DOIT-Q (Ours)}       & $\mathbf{+0.057(0.003)}$ & $0.517(0.035)$             & $\underline{1.983(0.001)}$ \\
    \bottomrule
  \end{tabular}
\end{table}
% END INLINE tables/results_ctrldna_per_cell

The final Ctrl-DNA Q operating point uses $M=52$, $K=20$, and
$s_\spl=64$, whereas DRAKES Q uses $M=52$, $K=8$, and $s_\spl=0$.
A preceding Ctrl-DNA eval sweep also considered $M=8,K=20,s_\spl=64$,
which obtained config-level $\Delta R=0.117$, and a
$M=32,K=8,s_\spl=0$ setting with $\Delta R=0.105$ at roughly three times
the runtime. These exploratory settings differ from the final Q comparison
and vary several parameters together. They illustrate alternative budget
allocations, not an isolated causal effect of branching. The endpoint
control shows that final selection alone already obtains much of the mean
specificity gain under the calibrated budget.

% BEGIN INLINE tables/results_ctrldna_evo
\begin{table}[t]
  \centering
  \renewcommand{\arraystretch}{1.25}
  \setlength{\tabcolsep}{14pt}
  \caption{\textbf{Evolutionary-search baselines on Ctrl-DNA.}
    We report $\Delta R$ and Motif Corr averaged across the three target cell lines for the four directed-evolution baselines reproduced from \cite{chen2025ctrldna}.
  The training-free and training-based comparisons appear in \cref{tab:results_ctrldna}.}
  \label{tab:results_ctrldna_evo}
  \begin{tabular}{@{}lcc@{}}
    \toprule
    Method & $\Delta R$\,$\uparrow$ & Motif Corr\,$\uparrow$ \\ \midrule
    AdaLead                        & $0.05(0.05)$              & $0.36(0.08)$ \\
    PEX                            & $0.09(0.02)$              & $0.23(0.04)$ \\
    BO                             & $0.01(0.02)$              & $0.11(0.11)$ \\
    CMA-ES                         & $0.01(0.01)$              & $0.08(0.05)$ \\
    \bottomrule
  \end{tabular}
\end{table}
% END INLINE tables/results_ctrldna_evo

\subsection{Wall-clock-matched Endpoint Control}
\label{app:ctrldna-search-controls}
Frozen Best-of-708 generates complete unguided sequences from the frozen
backbone and retains the best reward-scored endpoint in each group. We
calibrate its group size on the same A100 used for Q
($L=128,M=52,K=20,s_\spl=64$). At $5.51$ versus $5.43$ seconds per
output, the methods give specificity $0.108\pm0.0281$ versus
$0.119\pm0.063$, and Motif Corr $0.490\pm0.165$ versus $0.498\pm0.070$.
The means favor Q slightly, but the differences are modest and are not
interpreted as statistically significant. Training-based rows remain
cross-backbone references, separate from this inference-time comparison.

\FloatBarrier
\clearpage
\section{Ablation Study}
\label{sec:ablation}

We conduct ablation studies on the DRAKES HepG2 enhancer task to understand the behavior of D-DOIT.
We examine the guidance window $l^*$, guidance strength $\omega$ and effective
sample size (ESS), proposal and completion temperatures $\tau,\gamma$, and
candidate budget $M$. We also study completion count, direct-versus-rollout
value estimates, soft-versus-argmax selection, split position, and reward
trajectories (\cref{tab:diagnostic_completions,tab:diagnostic_proxy,tab:diagnostic_selection,tab:diagnostic_split,tab:diagnostic_suffix,tab:diagnostic_trajectory}).
The window, strength, temperature, and candidate-budget sweeps use a single
held-out seed unless noted otherwise; the additional diagnostics state their
sampling protocols separately. Final guidance reward $R$ is scored by the
sampling oracle and is distinct from held-out Pred-Activity. Completion-proxy
trajectories use the held-out evaluator but average scores of completed
partial states rather than the benchmark's within-seed batch median.
We re-run the best $\gamma$-sweep configuration across the $3$ main evaluation seeds and report it in \cref{tab:ablation_main}.
These studies show how D-DOIT balances reward optimization, sequence naturalness, and computational cost.

The candidate-budget sweep shows that increasing $M$ mainly improves functional utility, while naturalness metrics remain comparatively stable. The table surfaces three consequential design choices: the guidance window $l^{*}$ exhibits a phase transition for naturalness; the completion temperature $\gamma$ controls prediction-head completion quality; and the candidate budget $M$ controls utility.

% BEGIN INLINE tables/ablation_main
\begin{table}[ht]
  \centering
  \renewcommand{\arraystretch}{1.25}
  \caption{\textbf{Headline ablation findings.}
  Each row varies one hyperparameter around the DRAKES ablation reference setting. We report sweep-margin endpoints from a single-seed sweep; full sweeps appear in \cref{sec:ablation}.}
  \label{tab:ablation_main}
  \begin{tabular}{@{}lccc@{}}
    \toprule
    Axis & Setting & 3-mer Corr & Pred-Activity / ATAC (\%) \\
    \midrule
    $l^{*}$  & $64 \to 128$  & $-0.06 \to 0.894$ & $3.61 \to 4.77$ \,/\, $33.6 \to 72.7$ \\
    $\gamma$ & $1.2 \to 0.6$ & $0.846 \to 0.905$ & $4.53 \to 4.72$ \,/\, $72.7 \to 66.4$ \\
    $M$      & $8 \to 32$    & $+9\%$            & $+31\%$ \,/\, $+29$\,pp \\
    \bottomrule
  \end{tabular}
\end{table}
% END INLINE tables/ablation_main

\paragraph{Runtime analysis.}
\label{sec:ablation:runtime}
\Cref{sec:exp:runtime} reports the runtime comparison and its measurement
conditions. The runtime pattern follows the NFE accounting in
\cref{app:nfe-breakdown}: larger $M$ and earlier best-of-$K$ splitting increase
candidate-state evaluations and oracle-scored completions. Increasing $n_h$
also increases completion scoring, as quantified in
\cref{tab:diagnostic_completions}. LBOK-S is the lower-budget operating
point, whereas Q deliberately trades greater recurring inference cost for
quality; neither setting implies cost dominance over training-based methods.

\subsection{Effect of Guidance Window \texorpdfstring{$l^*$}{l*}}
\label{sec:ablation:lstar}

% BEGIN INLINE figures/fig_lstar_phase_transition
\begin{figure}[ht]
  \centering
  \includegraphics[width=1\textwidth]{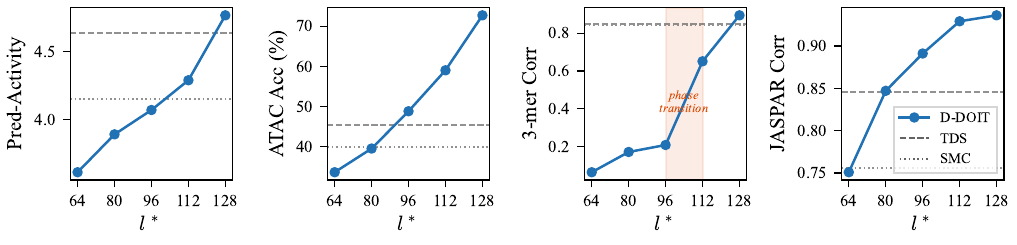}
  \caption{\textbf{Guidance-window sweep.}
  Expanding the guided window improves utility and sharply recovers naturalness between $l^*=96$ and $l^*=112$.}
  \label{fig:lstar_phase_transition}
\end{figure}
% END INLINE figures/fig_lstar_phase_transition

We study the effect of the guidance window $l^*$, the highest reverse-time index where Doob's $h$-transform is active in the discrete sampler.
A smaller $l^*$ saves $(L-l^*)\cdot M$ neural-function evaluations per sequence
but omits guidance at more heavily masked states, where limited revealed
context may affect the prediction-head proxy.
\Cref{fig:lstar_phase_transition} reports the four DRAKES metrics as we sweep $l^* \in \{10, 20, \ldots, 128\}$.

The results show a phase-transition-like behavior in naturalness.
All configurations with $l^* \leq 50$ produce negative 3-mer Pearson correlations.
A sharp jump occurs between $l^*{=}96$ (3-mer $= 0.21$) and $l^*{=}112$ ($0.65$).
At $l^*{=}128$, 3-mer reaches $0.894$, a $14\times$ improvement over $l^*{=}64$.
Utility metrics improve across the same range without a sharp transition.
In this sweep, broad guidance windows are associated with better global
sequence-composition statistics; this observation does not isolate a causal
mechanism or establish a requirement for other rewards.
Utility metrics, in contrast, recover from a narrower window.
We use $l^*{=}128$ as the default setting.

\subsection{Effect of Guidance Strength \texorpdfstring{$\omega$}{omega} and ESS Collapse}
\label{sec:ablation:omega}

We study the effect of the guidance strength $\omega$ in \eqref{eq:tilted-resample}.
A larger $\omega$ sharpens the transformed transition toward high-reward candidates.
A smaller $\omega$ preserves more of the pre-trained kernel.

% BEGIN INLINE figures/fig_ess_collapse
\begin{wrapfigure}{R}{0.42\linewidth}
  \centering
  \vspace{-1em}\includegraphics[width=\linewidth]{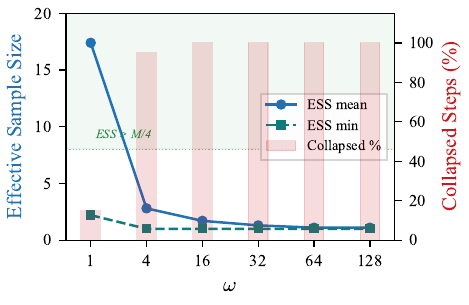}
  \vspace{-2em}
  \caption{\textbf{ESS diagnostic.}
  ESS and collapse fraction across $\omega$; weighted selection collapses to near-greedy at $\omega \geq 4$.}
  \label{fig:ess_collapse}
  \vspace{-1em}
\end{wrapfigure}
% END INLINE figures/fig_ess_collapse
\noindent\Cref{fig:ess_collapse} reports the per-step Effective Sample Size (ESS) of the candidate weights and the fraction of steps where ESS collapses below $M/4$.
The results show that ESS collapses to near $1$ at $\omega \geq 4$.
$95\%$ of guided steps have ESS $< M/4$ at $\omega{=}4$, and $100\%$ at $\omega \geq 16$.
A direct mechanism for concentration is reward-range amplification: $\omega{=}4$ amplifies a candidate reward range of $\approx 3$ nats to $\geq 12$ nats and saturates the softmax.
High-quality results remain achievable despite this collapse when proposal-level diversity (via $\tau$) still places a strong candidate in the pool.
This motivates joint tuning of $\omega$ and $\tau$, but does not imply that
larger $\tau$ recovers importance-weight accuracy. The concentration depends
on the reward distribution and $\omega/\beta$, so these diagnostic thresholds
cannot be transferred unchanged to other configurations.

\subsection{Effect of Proposal and Completion Temperatures \texorpdfstring{$\tau, \gamma$}{tau, gamma}}
\label{sec:ablation:tau_gamma}

We then study the joint effect of the proposal temperature $\tau$ in \eqref{eq:tempered-proposal} and the completion temperature $\gamma$ in \eqref{eq:h-hat-candidate}.
A larger $\tau$ flattens the proposal and produces more diverse candidates.
A smaller $\gamma$ sharpens the prediction-head completion toward the model mode.
\Cref{fig:tau_gamma_heatmap} reports the four DRAKES metrics across the joint sweep at $M{=}32$.

The results show a clear separation of roles.
Functional activity peaks at higher proposal diversity.
Marginal analysis shows Pred-Activity improves from $3.89$ at $\tau{=}2.5$ to $4.07$ at $\tau{=}5.0$, and ATAC-Acc from $53.7\%$ to $69.4\%$.
Naturalness peaks at sharper prediction-head completions.
Reducing $\gamma$ from $1.2$ to $0.6$ improves averaged 3-mer Corr from $0.846$ to $0.905$.
At the strongest $\gamma$-sweep configuration ($M{=}32$, $\tau{=}3.0$, $\gamma{=}0.6$), 3-mer Corr reaches $0.949 \pm 0.006$ across $3$ seeds, above DRAKES at $0.887$.
This sweep identifies the naturalness-preserving region.
The final held-out DRAKES Q row in \cref{tab:results_drakes} uses a larger eval-selected inference budget and trades some naturalness for much higher Pred-Activity and ATAC-Acc while remaining above the DRAKES 3-mer and JASPAR targets.
Naturalness metrics saturate around $\tau{=}3.5$--$4.0$, suggesting a trade-off between naturalness and accessibility at extreme proposal temperatures.
We use $\tau{=}3.0$ and $\gamma{=}0.6$ as the ablation default for naturalness-preserving sweeps.

% BEGIN INLINE figures/fig_tau_gamma_heatmap
\begin{figure}[!htbp]
  \centering
  \includegraphics[width=\linewidth]{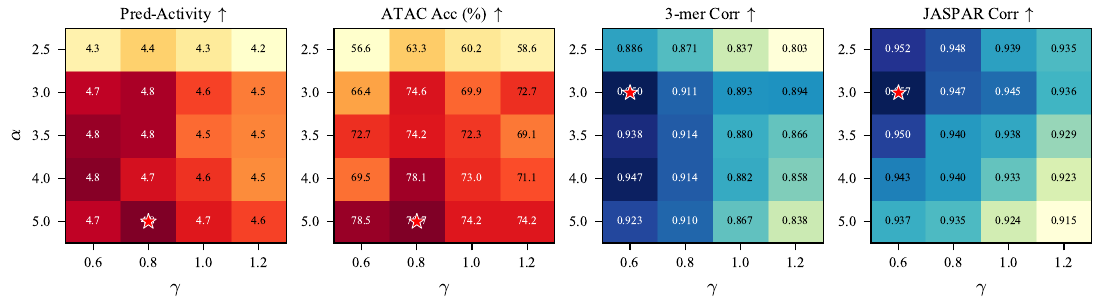}
  \caption{\textbf{Interaction of proposal and completion temperatures.}
    We sweep the proposal temperature $\tau$ and the completion temperature $\gamma$ at $M=32$.
    Naturalness peaks at sharper prediction-head completions (smaller $\gamma$), while functional activity peaks at higher proposal diversity (larger $\tau$).
  The red star marks the ablation reference configuration used before the final held-out Q budget selection.}
  \label{fig:tau_gamma_heatmap}
\end{figure}
% END INLINE figures/fig_tau_gamma_heatmap

\subsection{Effect of Candidate Budget \texorpdfstring{$M$}{M}}
\label{sec:ablation:m}

We study the effect of the Monte Carlo candidate budget $M$.
The candidate count controls per-step work and coverage of possible next
states. This sweep measures quality changes, not kernel error or convergence
to the original tilted distribution.
\Cref{fig:m_budget_ablation} reports the four DRAKES metrics across $M \in \{8, 16, 32\}$.

% BEGIN INLINE figures/fig_m_budget_ablation
\begingroup
\begin{figure}[!htbp]
  \centering
  \includegraphics[width=\linewidth]{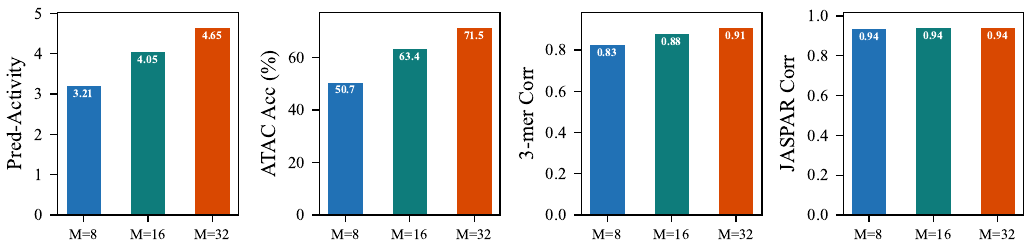}
  \caption{\textbf{Candidate-budget sweep.}
    We vary the Monte Carlo candidate budget $M$ on the DRAKES HepG2 task and report the same four metrics.
  Increasing $M$ gives direct gains in functional utility, while naturalness changes much less across the sweep.}
  \label{fig:m_budget_ablation}
\end{figure}
\endgroup
% END INLINE figures/fig_m_budget_ablation
\FloatBarrier

The results show that fewer candidates reduce functional utility but leave sequence naturalness almost unchanged.
Reducing $M$ from $32$ to $8$ drops Pred-Activity by $31\%$ and ATAC-Acc by $29\%$.
3-mer Corr drops by $9\%$ and JASPAR Corr by less than $1\%$.
The temperatures and the guidance window drive naturalness.
A larger candidate count gives more direct help to utility.
$M{=}32$ provides a strong per-step default within the ablation sweep.
The final held-out DRAKES Q row uses a larger budget, $M{=}52$ with $K{=}8$, for the no-budget-limit quality comparison in \cref{tab:results_drakes}.

\subsection{Completion Count and Matched-state Proxy Diagnostics}
\label{sec:ablation:proxy}

Completion count $n_h$ and candidate count $M$ address different aspects of
sampling: $M$ controls candidate-state coverage, whereas $n_h$ controls the
number of prediction-head completions averaged for each candidate in
\eqref{eq:multi-completion}. We reserve $N$ for sequence length.
\Cref{tab:diagnostic_completions} reports the full sampler with five seeds
and $3{,}200$ sequences per setting. Increasing $n_h$ from $1$ to $2$ raises
final guidance reward by $0.492$, at $41\%$ additional wall time; $n_h=3$
raises reward by $0.706$, at $81\%$ additional wall time. We retain $n_h=1$
as the practical compute--quality default. More completions reduce sampling
variance for the head expectation, but do not remove its factorization or
temperature mismatch with sequential rollouts. The completion-count sweep therefore measures a practical quality--cost
trade-off while retaining the same factorized head approximation.

% BEGIN INLINE tables/diagnostic_completions
\begin{table}[ht]
  \centering
  \caption{\textbf{Completion-count ablation on HepG2}. We use five seeds and $3{,}200$ sequences per setting. The score is final guidance reward $R$, not held-out Pred-Activity. Runtime increases are relative to $n_h=1$.}
  \label{tab:diagnostic_completions}
  \begin{tabular}{ccc}
    \toprule
    Completions $n_h$ & Final guidance reward $R$ & Wall-time increase \\
    \midrule
    $1$ & $5.272\pm0.057$ & baseline \\
    $2$ & $5.764\pm0.033$ & $41\%$ \\
    $3$ & $5.978\pm0.026$ & $81\%$ \\
    \bottomrule
  \end{tabular}
\end{table}
% END INLINE tables/diagnostic_completions

To examine the completion law, we save $640$ matched partial states at each
of three stages from $128$ DRAKES HepG2 source trajectories per seed over
five seeds. From each state, we independently generate $32$ direct
prediction-head completions and $32$ sequential reverse-process
continuations. For each route, we compute the $h$-value corresponding to
each completion's reward and average these values over the $32$ completions.
\Cref{tab:diagnostic_proxy} summarizes the pooled distributions of these
per-state averages. The fixed-base fraction is the percentage of nucleotides
already revealed in the partial state. W1 is Wasserstein-1 distance; KS is
the maximum absolute difference between the two empirical cumulative
distribution functions, not a significance-test result.

% BEGIN INLINE tables/diagnostic_proxy
\begin{table}[ht]
  \centering
  \caption{\textbf{Matched-state completion diagnostic.} At each stage, $640$ partial states come from $128$ source trajectories per seed over five seeds. Each state is independently completed $32$ times by the direct prediction head and $32$ times by sequential continuation. Moments and distances summarize the pooled distributions of per-state average $h$-values; parentheses denote variance, not uncertainty.}
  \label{tab:diagnostic_proxy}
  \small
  \begin{tabular}{ccccc}
    \toprule
    Fixed bases & Direct: mean (variance) & Rollout: mean (variance) & W1 & KS \\
    \midrule
    $0.78\%$ & $3.494\ (0.974)$ & $3.480\ (1.097)$ & $0.060$ & $0.045$ \\
    $33.39\%$ & $3.536\ (1.241)$ & $3.611\ (1.109)$ & $0.101$ & $0.050$ \\
    $66.34\%$ & $3.648\ (1.685)$ & $3.772\ (1.620)$ & $0.124$ & $0.058$ \\
    \bottomrule
  \end{tabular}
\end{table}
% END INLINE tables/diagnostic_proxy

The pooled means are similar, while W1 increases from $0.060$ at $0.78\%$
fixed bases to $0.124$ at $66.34\%$. KS remains between $0.045$ and $0.058$.
Thus the direct head gives a useful low-cost marginal approximation in this
diagnostic, with residual distributional shift that grows across the sampled
stages. These pooled summaries do not imply equality of the conditional
completion laws or exact sampling from the reward-tilted target.

\subsection{Soft Selection versus Argmax}
\label{sec:ablation:argmax}

\Cref{tab:diagnostic_selection} compares soft selection in
\eqref{eq:tilted-resample} with direct argmax using $M=8$, $\tau=3$, $K=4$,
and $\beta=0.2$, over five seeds and $3{,}200$ sequences per setting.
This diagnostic is separate from the high-budget Q operating point.
As $\omega$ rises from $1$ to $8$, mean ESS falls from $1.877$ to $1.143$
and top-candidate selection rises from $77.53\%$ to $95.72\%$.
At $\omega=1$, soft selection has lower final guidance reward than argmax
($4.894\pm0.055$ versus $5.154\pm0.027$); at $\omega=8$, it has higher
reward ($5.237\pm0.051$). The behavior is therefore close to stochastic
best-of-$M$ selection, rather than broad resampling, but not identical to
argmax. The remaining $4.28\%$ non-top selections at $\omega=8$ can change
subsequent states and candidate pools. Low ESS describes local weight
concentration, not final-sequence diversity or target-distribution accuracy.

% BEGIN INLINE tables/diagnostic_selection
\begin{table}[ht]
  \centering
  \caption{\textbf{Soft selection versus direct argmax on HepG2.} We use $M=8$, $\tau=3$, $K=4$, and $\beta=0.2$, with five seeds and $3{,}200$ sequences per setting. ESS and top-candidate frequency describe guided-step selection; rewards are final guidance reward $R$. Direct argmax does not depend on $\omega$.}
  \label{tab:diagnostic_selection}
  \small
  \setlength{\tabcolsep}{4pt}
  \begin{tabular}{cccccc}
    \toprule
    $\omega$ & Mean ESS & \makecell{Largest\\normalized weight} & \makecell{Top-candidate\\selection rate} & \makecell{Soft\\reward $R$} & \makecell{Argmax\\reward $R$} \\
    \midrule
    $1$ & $1.877\pm0.020$ & $0.775$ & $77.53\%$ & $4.894\pm0.055$ & $5.154\pm0.027$ \\
    $4$ & $1.237\pm0.004$ & $0.927$ & $92.69\%$ & $5.183\pm0.038$ & $5.154\pm0.027$ \\
    $8$ & $1.143\pm0.002$ & $0.957$ & $95.72\%$ & $5.237\pm0.051$ & $5.154\pm0.027$ \\
    \bottomrule
  \end{tabular}
\end{table}
% END INLINE tables/diagnostic_selection

\subsection{Split Time and Reward Trajectories}
\label{sec:ablation:split}

\paragraph{Earlier branching improves reward at greater cost.}
\Cref{tab:diagnostic_split} varies only $s_\spl$, keeping $M=9$, $K=4$, and
$n_h=1$ fixed. Splitting at update $0$ gives final guidance reward
$5.890\pm0.028$ and a best-of-$4$ gain of $1.605\pm0.019$; splitting at
update $105$ gives $5.272\pm0.057$ and $0.959\pm0.020$, respectively.
The earlier split costs $793.6\pm2.2$ seconds per job, compared with
$326.4\pm0.5$ seconds for the late split. Late branching retains a useful
selection gain while sharing more of the trajectory computation, but it
cannot revise tokens already fixed before the split.

% BEGIN INLINE tables/diagnostic_split
\begin{table}[ht]
  \centering
  \caption{\textbf{Controlled split-position ablation with $M=9$, $K=4$, and $n_h=1$ fixed.} Only $s_\spl$ varies. Both final reward and best-of-$4$ selection gain use the guidance reward $R$, not held-out Pred-Activity.}
  \label{tab:diagnostic_split}
  \begin{tabular}{cccc}
    \toprule
    Split update $s_\spl$ & Final guidance reward $R$ & Best-of-$4$ gain & Wall time/job (s) \\
    \midrule
    $0$ & $5.890\pm0.028$ & $1.605\pm0.019$ & $793.6\pm2.2$ \\
    $32$ & $5.855\pm0.035$ & $1.545\pm0.030$ & $653.8\pm1.5$ \\
    $64$ & $5.729\pm0.029$ & $1.435\pm0.019$ & $510.4\pm1.9$ \\
    $105$ & $5.272\pm0.057$ & $0.959\pm0.020$ & $326.4\pm0.5$ \\
    \bottomrule
  \end{tabular}
\end{table}
% END INLINE tables/diagnostic_split

\paragraph{Diversity remains in the unresolved suffix.}
For saved partial states from the frozen reverse process, we sample $32$
independent suffix completions and compute pairwise Hamming distance only
on positions that were still masked (\cref{tab:diagnostic_suffix}).
The distance is $0.710\pm0.010$ with $25.43\%$ of bases fixed, and remains
$0.661\pm0.011$ with $90.38\%$ fixed. These results demonstrate substantial
suffix diversity conditional on a shared prefix, not determinism of early
reveals. Late best-of-$K$ exploits this remaining search space; the complete
D-DOIT--TDS comparison does not isolate branching's contribution.

% BEGIN INLINE tables/diagnostic_suffix
\begin{table}[ht]
  \centering
  \caption{\textbf{Suffix diversity under the frozen reverse process.} For each saved partial sequence, $32$ independent suffix completions are sampled. Hamming distance is the mean pairwise fraction of differing bases restricted to positions that were still masked in the shared partial state.}
  \label{tab:diagnostic_suffix}
  \begin{tabular}{ccc}
    \toprule
    Fixed positions & Still-masked positions & Remaining-position Hamming distance \\
    \midrule
    $25.43\%$ & $74.57\%$ & $0.710\pm0.010$ \\
    $50.38\%$ & $49.62\%$ & $0.689\pm0.010$ \\
    $90.38\%$ & $9.62\%$ & $0.661\pm0.011$ \\
    \bottomrule
  \end{tabular}
\end{table}
% END INLINE tables/diagnostic_suffix

\paragraph{Completion-proxy progress is nonmonotone.}
\Cref{tab:diagnostic_trajectory} scores partial states after filling masked
positions with the frozen model's deterministic argmax predictions.
The held-out Enformer scores are averaged over $640$ sequences per seed
and then over five seeds. They are neither guidance reward $R$ nor the
batch-median Pred-Activity metric. D-DOIT's completion-proxy score rises
from $2.067$ at update $1$ to $2.328$ at update $32$, falls to $1.758$ at
update $96$, and reaches $3.670$ at update $128$; frozen sampling declines
from $2.053$ to $0.371$. This trajectory shows higher endpoint activity
without monotone intermediate progress. An argmax-completion score is not
the rollout $h$-function or a guarantee about future reward.

% BEGIN INLINE tables/diagnostic_trajectory
\begin{table}[ht]
  \centering
  \caption{\textbf{Held-out enformer completion-proxy trajectory.} At each update, masked positions are filled with the frozen model's deterministic argmax predictions before scoring. Scores are averaged over $640$ sequences per seed and then over five seeds, rather than using the batch-median Pred-Activity aggregation.}
  \label{tab:diagnostic_trajectory}
  \begin{tabular}{lccccc}
    \toprule
    Update & $1$ & $32$ & $64$ & $96$ & $128$ \\
    \midrule
    D-DOIT & $2.067$ & $2.328$ & $1.763$ & $1.758$ & $3.670$ \\
    Frozen sampling & $2.053$ & $1.561$ & $1.058$ & $0.716$ & $0.371$ \\
    \bottomrule
  \end{tabular}
\end{table}
% END INLINE tables/diagnostic_trajectory

From update $32$ onward, the argmax-completed D-DOIT sequences are $100\%$
unique, and their mean pairwise Hamming distance increases from $0.726$ to
$0.749$. This diversity measure concerns completed sequences across the
sample batch, unlike the remaining-position conditional diversity above.
The observed late-stage reward improvement is therefore compatible with
maintained diversity in this diagnostic, rather than implying that guidance
universally preserves diversity.

\FloatBarrier
\section{Experimental Details of Protein Inverse Folding}
\label{app:exp:protein}
Protein inverse folding tests amino-acid sequence design conditional on a
target backbone, complementing the regulatory DNA experiments, and the full
cross-method comparison, including all metrics and seed-level SDs, is in
\cref{tab:protein-results-full}. We follow
the protein inverse-folding protocol of DRAKES~\citep{wang2025drakes}, using
the same reward oracle and evaluation metrics. D-DOIT guides the frozen
conditional generator through reward-value queries without updating its
parameters.

\paragraph{Protocol and aggregation.}
We evaluate all $12$ test protein backbones with three independent sampling
seeds and generate $128$ sequences per backbone and seed. The configuration
$M=52$, $K=8$, and $n_h=1$ (one completion per candidate) is fixed before
test-set evaluation, with no tuning after inspecting test results. Each
metric is computed separately for each seed under the DRAKES protocol;
D-DOIT entries report the mean and sample standard deviation across those
three seed-level metrics. The uncertainty describes sampling-seed
variation, not variation across sequences or backbones.

\paragraph{Metrics.}
Pred-ddG is the predicted stability score, with higher values preferred
under the benchmark's sign convention. We report its median and the
percentage of sequences with predicted ddG greater than zero. Structural
fidelity is measured by self-consistency RMSD (scRMSD), comparing the
predicted structure of a generated sequence with its target backbone; we
report the median in \AA\ and the percentage below $2$\,\AA.
Joint success requires both positive predicted ddG and scRMSD below
$2$\,\AA. These are predictor-based measures rather than experimental
stability or folding measurements.

% BEGIN INLINE tables/protein_results
\begin{table}[ht]
  \centering
  \caption{\textbf{Protein inverse folding on the DRAKES benchmark.}
    D-DOIT entries are mean $\pm$ sample SD across seed-level metrics.
    Other rows and uncertainties are from DRAKES Table~2
    \citep{wang2025drakes}, not rerun.
    Bold marks the best mean per column, not statistical significance.}
  \label{tab:protein-results-full}
  \small
  \setlength{\tabcolsep}{4pt}
  \resizebox{\linewidth}{!}{%
    \begin{tabular}{@{}lccccc@{}}
      \toprule
      Method & \makecell{Pred-ddG median $\uparrow$} & \makecell{ddG $>0$ (\%) $\uparrow$} & \makecell{scRMSD median (\AA) $\downarrow$} & \makecell{scRMSD $<2$\,\AA\ (\%) $\uparrow$} & \makecell{Success (\%) $\uparrow$} \\
      \midrule
      Pre-trained & $-0.544\pm0.037$ & $36.6\pm1.0$ & $0.849\pm0.013$ & $90.9\pm0.6$ & $34.4\pm0.5$ \\
      CG & $-0.561\pm0.045$ & $36.9\pm1.1$ & $0.839\pm0.012$ & $90.9\pm0.6$ & $34.7\pm0.9$ \\
      SMC & $0.659\pm0.044$ & $68.5\pm3.1$ & $0.841\pm0.006$ & $93.8\pm0.4$ & $63.6\pm4.0$ \\
      TDS & $0.674\pm0.086$ & $68.2\pm2.4$ & $\mathbf{0.834\pm0.001}$ & $\mathbf{94.4\pm1.2}$ & $62.9\pm2.8$ \\
      CFG & $-1.186\pm0.035$ & $11.0\pm0.4$ & $3.146\pm0.062$ & $29.4\pm1.0$ & $1.3\pm0.4$ \\
      \midrule
      D-DOIT & $\mathbf{0.870\pm0.046}$ & $\mathbf{74.8\pm3.1}$ & $0.872\pm0.009$ & $92.5\pm1.3$ & $\mathbf{68.1\pm3.5}$ \\
      \bottomrule
    \end{tabular}%
  }
\end{table}
% END INLINE tables/protein_results

\paragraph{Results and comparison boundary.}
D-DOIT achieves the highest median Pred-ddG ($0.870\pm0.046$),
ddG-positive fraction ($74.8\pm3.1\%$), and joint success
($68.1\pm3.5\%$) among the six methods. Joint success exceeds the
published SMC and TDS means by $4.5$ and $5.2$ percentage points,
respectively. TDS retains better structural fidelity: median scRMSD is
$0.834$\,\AA\ versus $0.872$\,\AA, and the fraction below $2$\,\AA\ is
$94.4\%$ versus $92.5\%$.
All non-D-DOIT entries are taken directly from DRAKES Table~2 and were not
rerun in this comparison; their reported uncertainties are retained.
The shared benchmark oracle and metrics permit a literature comparison,
not a matched-compute or matched-implementation claim. The results extend the computational evaluation to protein sequence design;
they do not establish experimentally measured stability or folding.

\FloatBarrier
\section{Licenses for Existing Assets}
\label{app:asset-licenses}

We use pre-existing assets for non-commercial academic research and credit their creators through Sections~J and~K. The Masked Discrete Diffusion Model reference implementation of \citet{sahoo2024simple} carries Apache-2.0; the Ctrl-DNA repository of \citet{chen2025ctrldna}, Enformer implementation of \citet{avsec2021effective}, and HyenaDNA repository of \citet{nguyen2023hyenadna} also carry Apache-2.0 licenses.
The HepG2 enhancer corpus from \citet{gosai2023machine} is distributed through the paper's data-availability resources, and the Nature article carries CC BY~4.0. The ENCODE Data Use Policy governs chromatin-accessibility data~\citep{encode2012integrated}; JASPAR~2022~\citep{castro2022jaspar} is open access and its Bioconductor data package is GPL-2.
The public DRAKES repository associated with \citet{wang2025drakes} had no top-level software license in the GitHub metadata at access time. We use DRAKES only as a cited benchmark, baseline, oracle, and checkpoint source for academic evaluation, and do not redistribute its code or checkpoints.

\let\section\appendixoriginalsection

% \clearpage
% \input{checklist}

\end{document}